\documentclass{article}
\usepackage[preprint,nonatbib]{neurips_2026}
\usepackage[round]{natbib}

\usepackage[utf8]{inputenc}
\usepackage[T1]{fontenc}
\usepackage{fix-cm}
\usepackage{hyperref}
\hypersetup{hidelinks}
\usepackage{url}
\usepackage{booktabs}
\usepackage{array}
\usepackage{amsmath}
\usepackage{amssymb}
\usepackage{amsthm}
\usepackage{thm-restate}
\usepackage{amsfonts}
\usepackage{mathtools}
\usepackage{nicefrac}
\usepackage{microtype}
\usepackage{cancel}
\usepackage{textcomp}
\usepackage{xcolor}
\usepackage{graphicx}
\usepackage{subcaption}
\usepackage{needspace}
\usepackage{float}
\usepackage{cleveref}

\newcommand{\Elog}{\mathcal{E}_{\mathrm{logmap}}}
\newcommand{\Rlog}{\mathcal{R}_{\mathrm{logmap}}}
\newcommand{\Eexp}{\mathcal{E}_{\mathrm{expmap}}}
\newcommand{\Rexp}{\mathcal{R}_{\mathrm{expmap}}}

\DeclareMathOperator{\diag}{diag}
\DeclareMathOperator{\tr}{tr}

\newcommand{\R}{\mathbb{R}}
\newcommand{\M}{\mathcal{M}}

\newcommand{\X}{\mathcal{X}}
\newcommand{\Z}{\mathcal{Z}}
\newcommand{\Ss}{\mathcal{S}}

\newcommand{\x}{\mathbf{x}}
\newcommand{\y}{\mathbf{y}}
\newcommand{\s}{\mathbf{s}}
\newcommand{\e}{\mathbf{e}}
\newcommand{\z}{\mathbf{z}}
\newcommand{\p}{\mathbf{p}}
\newcommand{\uu}{\mathbf{u}}
\newcommand{\zero}{\mathbf{0}}
\newcommand{\one}{\mathbf{1}}

\newcommand{\circlenum}[1]{\textcircled{\scriptsize #1}}

\newtheorem{proposition}{Proposition}
\newtheorem{lemma}[proposition]{Lemma}
\newtheorem{definition}[proposition]{Definition}

\begin{document}

\title{Learning Normal Coordinates on Riemannian Manifolds with Flow Charts}
\author{%
  Edmond Cunningham \\
  University of Massachusetts Amherst \\
  \texttt{edmondcunnin@umass.edu}
}
\date{}
\maketitle

\begin{abstract}
Riemannian manifolds are powerful models for problems with non-Euclidean structure.
Recent work in geometric machine learning has focused on training deep neural networks to learn Riemannian metrics from high-dimensional data.
However, many geometric algorithms require repeated evaluations of the exponential and logarithmic maps, whose high computational cost restricts the utility of learned Riemannian manifolds in practice.
We propose to address this cost by training a conditional invertible neural network, called a flow chart, to approximate normal coordinates around points on a given Riemannian manifold.
With this approach, the exponential and logarithmic maps can be approximated by fast Euclidean operations in the model's learned coordinate space.
Our main contribution is an efficient loss function for training flow charts that is nonnegative and, under suitable domain assumptions, is zero precisely when the chart is normal.
Across sixteen synthetic and data-derived geometries, flow charts are competitive with a specialized neural geodesic interpolator for logarithmic map estimation, while also being able to estimate the exponential map.
When pretrained, flow charts provide a median speedup of $60$--$70\times$ over numerical baselines at comparable logarithmic map estimation error, and their learned coordinates improve the accuracy of a numerical geodesic solver.
\end{abstract}

\section{Introduction}

\begin{figure*}[t]
    \centering
    \includegraphics[width=\textwidth,height=0.31\textheight,keepaspectratio]{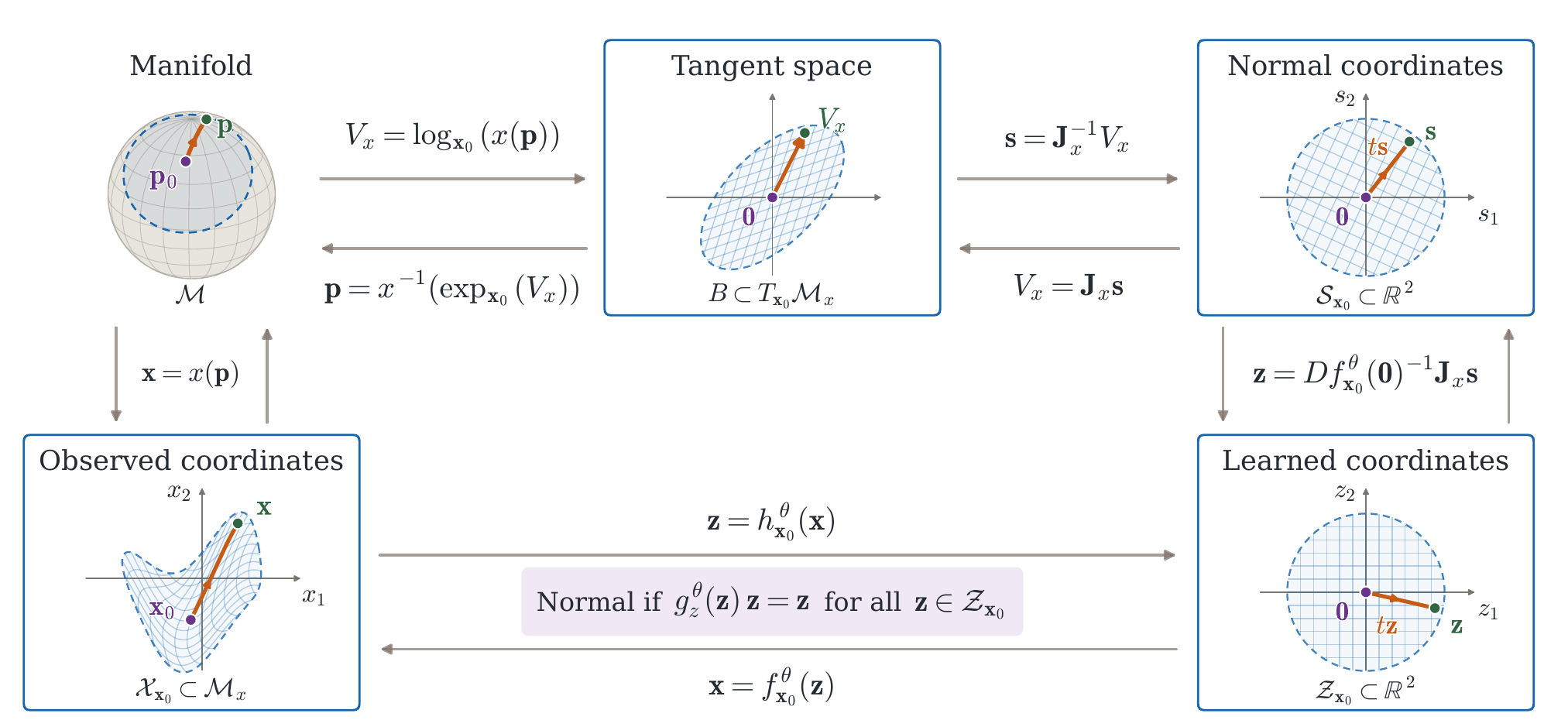}
    \caption{Illustration of normal coordinate charts and our method, using our notation in \Cref{sec:background} and \Cref{sec:method}.
    The purple and green symbols correspond to the same points $\p_0$ and $\p$ on the manifold in the top left, and the red curve is the geodesic between them.
    The tangent-space and coordinate plots depict the same blue region of the manifold in different numerical representations.
    The top row shows the standard approach to defining normal coordinates through the logarithmic map and the bottom shows our approach.
    Our method introduces flow charts and trains them to satisfy the identity in purple to learn normal coordinates.}
    \label{fig:normal-coordinates-flow-chart}
\end{figure*}

Many real-world problems can be modeled effectively with Riemannian manifolds whose metric encodes feasible state transitions \citep{diepeveen2024riemannian,wurgaft2026manifold,gaga}.
These transitions are represented by shortest paths, called geodesics, so that problems like trajectory inference \citep{kapusniak2024metric} and probabilistic inference \citep{girolami2011riemann,williams2025geodesic} can be expressed through geodesic computations.
When the geometry is known a priori through domain constraints, geodesics can be computed in closed form or simulated numerically \citep{lee2018introduction,miolane2020geomstats}.
However, recent work in geometric machine learning has focused on learning metrics directly from data \citep{gaga,jones2026computing,bamberger2026metricmatching} using neural networks.
While these methods offer exciting new avenues for geometric deep learning \citep{bronstein2021geometric}, these neural metrics present computational challenges that limit their downstream utility.

The two computations that are fundamental to Riemannian algorithms are the exponential and logarithmic maps.
These maps play the role of addition and subtraction of points on a manifold, so Riemannian extensions of classical algorithms such as PCA and regression \citep{fletcher2011geodesic,fletcher2004principal} evaluate them repeatedly.
On general manifolds, neither map has a closed form and both must be solved for numerically.
The exponential map is often computed as the solution to a differential equation, like the geodesic equation or Hamilton's equations \citep{hairer2006geometric}, both of which can be simulated efficiently with automatic differentiation \citep{williams2025geodesic}.
However, the simulation updates depend on the metric's derivative, which can be unreliable for learned metrics whose value may change rapidly \citep{arvanitidis2019fast}.
The logarithmic map is the inverse of the exponential map, and requires solving for the geodesic interpolation between two points at a time \citep{hennig2014probabilistic,rygaard2026georce}.
Although this cost can be amortized with neural networks \citep{kapusniak2024metric,kelshaw2024computing}, these methods take both endpoints as input and so they are not able to evaluate the exponential map.
We explore a different approach to computing both maps based on learning a coordinate system that simplifies their computations, called Riemann normal coordinates.

Normal coordinates are an important tool for calculations in Riemannian geometry \citep[page 78]{lee2018introduction}.
Within a suitable neighborhood of a point, geodesics through that point become straight lines in coordinate space, and so the exponential and logarithmic maps can be computed in closed form with addition and subtraction.
However, normal coordinates are most often defined in terms of the exponential and logarithmic maps, and methods for computing them rely on the numerical approximations of these maps \citep{brun2005fast,lin2008riemannian}.
We build on an alternative characterization of normal coordinates that does not depend on either map.

We propose a novel approach for parameterizing and learning normal coordinates on a Riemannian manifold using \emph{flow charts}, which are conditional invertible neural networks that represent coordinate charts.
Our main contribution is an efficient loss function for training flow charts whose value is nonnegative and equals zero over a geodesic ball if and only if the chart is normal.
In our experiments, we first show that flow charts trained with the normal loss can recover the logarithmic map accurately at fixed base points on both synthetic manifolds and geometries derived from data.
We then show that charts trained over many base points generalize to unseen ones and achieve median logmap query speedups of $60\times$ over GP \cite{arvanitidis2019fast} and $72\times$ over GEORCE \cite{rygaard2026georce} across sixteen geometries, and exponential-map query speedups of $1.4$--$8.2\times$ on the seven high-dimensional geometries at comparable error.
Finally, we show that GEORCE can compute more accurate logarithmic maps if applied in the flow chart's learned coordinate space rather than in data coordinates.

\section{Related work}
\label{sec:related}
Our work is related to methods that learn parametric curves to approximate geodesic interpolations between two points in the setting where the derivatives of the metric are assumed to be unreliable \citep{arvanitidis2019fast}.
These curve parameterizations include splines \citep{pmlr-v151-arvanitidis22b,karczewski2026spacetime}, the posterior mean of a Gaussian process \citep{arvanitidis2019fast}, and neural networks \citep{kapusniak2024metric}.
While these methods do not learn normal coordinates directly, they offer normal coordinate estimates as their initial velocities.
This approach is also how classical data-driven methods for learning normal coordinates operate, where solutions are derived from shortest paths on a graph over data points \citep{brun2005fast,lin2008riemannian}.
These approaches are different from our method in that we learn normal coordinates directly whereas the other methods recover them as a byproduct.

More broadly, flow charts are related to methods that use neural networks to amortize the cost of expensive geometric computations.
The most relevant method is metric flow matching (MFM), which parameterizes the geodesic between two endpoints with a neural network \citep{kapusniak2024metric}.
This method learns a geodesic interpolant by minimizing the energy of a parametric curve between the two endpoints under the given metric.
Other methods parameterize the geodesic distance function instead.
\citet{kelshaw2024computing} and \citet{li2024riemannian} train a network to satisfy the eikonal equation of a given metric.
These methods are able to recover the logarithmic map using the gradient of the distance function.
Flow charts are also neural network based, but they parameterize normal coordinate charts, which are a more general structure that naturally model the exponential map, logarithmic map and Riemannian distance at once.

Like flow charts, there are other methods that represent a geometry for data using invertible neural network architectures.
The pullback metric through a normalizing flow has been used as a geometric model of data \citep{diepeveen2024pulling,diepeveen2024score} and is useful for counterfactual explanations \citep{dombrowski2021diffeomorphic} and interpolations \citep{fadel2021principled}.
However, each of these methods represents its geometry with a single chart.
Multi-chart methods instead cover the data with a finite mixture of flow-based charts \citep{kalatzis2021multichart,sidheekh2022vq,paik2022atlas,yu2026learning}, but are limited in the geometries they can represent because they use a fixed number of charts.
Our method differs from these approaches in that it is defined on a given Riemannian manifold, instead of defining its own geometry, and parameterizes a continuous family of charts that can be expanded at any point on the manifold, instead of a discrete set of points.

\section{Background}
\label{sec:background}
In this section we review basic notation and terminology that will be used in the paper.
We intentionally emphasize the coordinate dependence of geometric objects because our method will parameterize a coordinate system.
See \citet{Lee00,lee2018introduction} for a reference on differential geometry.

Let $\M$ be an $n$-dimensional smooth manifold that represents the abstract space that contains data.
A \emph{coordinate chart} is a bijective map from a neighborhood of an abstract point $\p \in \M$ to a subset of $\R^n$ on which geometric objects admit numeric representations.
We assume that the region of $\M$ that data is observed in can be represented by one coordinate system, $x \colon \M \to \M_x = \R^n$, and $\x = x(\p)$ denotes the point $\p \in \M$ in the $x$-coordinate system, so that a dataset contains the $x$-coordinates of a collection of points.
Subscripts denote the coordinate system in which a geometric object on $\M$ is represented.
For example, $V_x \in \R^n$ represents the components of a tangent vector $V \in T_{\p} \M$, and $\omega_x \in \R^n$ represents the components of a covector $\omega \in T_{\p}^* \M$.
The notation $Df(\x) = \tfrac{\partial f(\x)}{\partial x}$ denotes the Jacobian matrix of $f$ at $\x$.

A \emph{Riemannian metric} is a point-dependent inner product between two tangent vectors represented by a symmetric positive-definite matrix $g_x(\x) \in$ $\mathrm{SPD}(n)$ at each $\x \in \M_x$.
The \emph{cometric} defines an inner product between cotangent vectors and is represented by $g_x(\x)^{-1}$.
The pair $(\M, g)$ is called a Riemannian manifold, on which geometric objects like shortest paths, called \emph{geodesics}, are uniquely defined by $g$.

A curve $\x_{(\cdot)}\colon \mathbb{R} \to \M_x$ starting at an initial point $\x_0 \in \M_x$ with velocity $V_x \in \R^n$ is a geodesic if it satisfies an ODE called the geodesic equation.
The solution of this ODE is called the \emph{exponential map}.
Given the initial conditions $(\x_0, V_x)$, the exponential map returns the point obtained from simulating the geodesic equation for $1$ unit of time, $\x_1 = \exp_{\x_0}(V_x)$.
Rescaling the initial velocity with some $t \in \R$ obtains other points on the geodesic, $\x_t := \exp_{\x_0}(tV_{x})$.

Let $B = \{V_x \in \R^n \colon V_x^\top g_x(\x_0)V_x < \varepsilon^2\} \subset T_{\x_0}\M_x$ be a ball of $g$-radius $\varepsilon$ that is small enough for $\exp_{\x_0}$ to be a diffeomorphism on $B$.
Then $\X_{\x_0} := \exp_{\x_0}(B) \subset \M_x$ is a \emph{geodesic ball} on $\M_x$ centered at $\x_0$.
The \emph{logarithmic map}, $\log_{\x_0}\colon \X_{\x_0} \to B$, is the inverse of $\exp_{\x_0}$ on $\X_{\x_0}$.
If a point $\x$ is outside of the geodesic ball, there may be many geodesics connecting $\x_0$ and $\x$, in which case the logarithmic map is not unique.

\label{sec:background-riemann-normal-coordinates}
\emph{Riemann normal coordinates} (or ``normal coordinates'') are a coordinate system in which geodesics through the base point are radial lines from the origin.
Let $\X_{\x_0} \subset \M_x$ be a geodesic ball centered at $\x_0 \in \M_x$ and let $\mathbf{J}_x \in \R^{n\times n}$ be a $g$-orthogonal matrix at $\x_0$, which satisfies $\mathbf{J}_x^\top g_x(\x_0)\mathbf{J}_x = I_n$.
A \emph{normal coordinate chart} on $\X_{\x_0}$ maps $\x$ to $\s = \mathbf{J}_x^{-1}V_x \in \Ss \subset \R^n$, where $V_x = \log_{\x_0}(\x)$, so that $\x_0$ maps to $\zero$.
See \Cref{fig:normal-coordinates-flow-chart}.
Normal coordinate charts are invertible on their domains, with inverse $\s \mapsto \exp_{\x_0}(\mathbf{J}_x\s)$, and are uniquely defined up to the choice of $\mathbf{J}_x$.
Because $D\exp_{\x_0}(\zero)$ is the identity map, the Jacobian of the inverse coordinate transformation at the origin is $\mathbf{J}_x$.

\section{Method}
\label{sec:method}
We present a novel approach for learning normal coordinate charts on Riemannian manifolds.
Our method represents coordinate charts using a conditional invertible neural network and is trained with a loss that renders the flow chart approximately normal.

\subsection{Flow chart}
\label{sec:flow-chart}
We parameterize coordinate charts using a conditional invertible neural network, called a \emph{flow chart}.
\begin{definition}[Flow chart]
    \label{def:flow-chart}
    Let $\M_x \subset \R^n$ be the coordinate representation of the manifold $\M$ in the $x$ coordinate system.
    A \emph{flow chart} is a map $h^\theta_{(\cdot)}(\cdot)\colon \M_x \times \M_x \to \R^n$ with parameters $\theta$ satisfying two conditions.
    \begin{enumerate}
        \item For every base point $\x_0 \in \M_x$, the map $h^\theta_{\x_0}\colon \M_x \to \R^n$ is a diffeomorphism, with inverse denoted by $f^\theta_{\x_0} := (h^\theta_{\x_0})^{-1}$.
        \item Every base point maps to the origin of its own chart, $h^\theta_{\x}(\x) = \zero$ for all $\x \in \M_x$.
    \end{enumerate}
    We denote $\z = h^\theta_{\x_0}(\x)$ as the coordinate representation of $\x$ under the chart $h^\theta_{\x_0}$.
\end{definition}
In practice, we use normalizing flow architectures with fast inverses to parameterize flow charts \citep{rezende_variational_nodate,papamakarios2021normalizing,durkan2019neural}.

Flow charts and normalizing flows are closely related in that they both learn coordinate spaces where expensive computations on $\M_x$ are simplified.
From a geometric perspective, normalizing flows learn a coordinate chart so that the pushforward density of a complicated target probability distribution equals the density of a simple prior.
This way, sampling from the target reduces to sampling from the prior.
Similarly, given a target metric $g_x$, flow charts will be trained so that the coordinate representation, $g_z^\theta$, satisfies an identity that implies the chart is normal.
For a flow chart that is centered at $\x_0 \in \M_x$, its coordinate representation of the metric at a point $\x \in \M_x$ is given by the following tensor transformation rule \citep{lee2018introduction}:
\begin{align}
    \label{eq:metric-change-of-variables}
    g_z^\theta(\z) = Df_{\x_0}^\theta(\z)^\top g_x(\x)Df_{\x_0}^\theta(\z)
\end{align}
Note that the particular coordinate representation depends on the base point $\x_0$ and that we do not write the dependence of $g_z^\theta(\z)$ on $\x_0$ for brevity.

\subsection{Normal flow chart}
\label{sec:normal-flow-chart}
Our main observation is that normal coordinate charts can be identified as flow charts satisfying the identity in the following proposition.

\begin{restatable}[Normal chart identity]{proposition}{gausslemma}
    \label{prop:normal-chart-identity}
    Let $(\M, g)$ be a Riemannian manifold with global coordinate system $x\colon \M \to \M_x$, $\x_0 \in \M_x$ be a point, $\X_{\x_0} \subset \M_x$ be a geodesic ball centered at $\x_0$, $\mathcal{U}_x \subseteq \X_{\x_0}$ be an open set containing $\x_0$, and $h_{\x_0}\colon \mathcal{U}_x \to \R^n$ be a coordinate chart, where $\z = h_{\x_0}(\x)$ represents the coordinate of $\x \in \M_x$, $h_{\x_0}(\x_0) = \zero$, and $\Z_{\x_0}:= h_{\x_0}(\mathcal{U}_x)$ is star-shaped with respect to the origin.
    Let $g_z(\z) = Df_{\x_0}(\z)^\top g_x(\x)Df_{\x_0}(\z)$ denote the metric's $z$ components.
    We say that $h_{\x_0}$ is normal on $\mathcal{U}_x$ if it is the restriction to $\mathcal{U}_x$ of a normal coordinate chart on $\X_{\x_0}$.
    Then,
    \begin{align}
        h_{\x_0}(\cdot)\text{ is normal on $\mathcal{U}_x$ if and only if  }g_z(\z)\z = \z,\quad \forall \z\in \Z_{\x_0}
        \label{eq:normal-chart-identity}
    \end{align}
\end{restatable}
This identity is Gauss's lemma written in normal coordinates and serves as a definition of normal coordinates in physics \citep{muller1999closed}, but, to our knowledge, has not been used in a machine learning context.
See \Cref{app:normal-chart-identity-proof} for a proof.

The advantage of this characterization of normality is that the identity can be checked with any flow chart and only requires matrix-vector products with the metric and Jacobian-vector products with the flow chart.
Additionally, as illustrated by \Cref{fig:normal-coordinates-flow-chart}, when a flow chart satisfies the normal chart identity, the exponential and logarithmic maps can be recovered in closed form on the domain where the identity holds:
\begin{align}
    \exp_{\x_0}(V_x) = f_{\x_0}^\theta(Dh_{\x_0}^\theta(\x_0)V_x), \qquad \log_{\x_0}(\x) = Df_{\x_0}^\theta(\zero)h_{\x_0}^\theta(\x)
\end{align}

We train a flow chart to satisfy \Cref{eq:normal-chart-identity} at a given $\x_0$ by optimizing its parameters to minimize the residual of the identity.

\begin{restatable}[Normal loss]{proposition}{normalityerror}
    \label{proposition:normal-loss}
    Let $(\M,g)$ be a Riemannian manifold, $h_{\x_0}^\theta$ be a flow chart with parameters $\theta$, $\mathcal{U}_x \subset$ $\X_{\x_0}$ be an open set containing $\x_0$, and $\rho_x(\cdot|\x_0)$ be a probability density restricted to $\mathcal{U}_x$, and assume $\rho_x(\x|\x_0) > 0$ over the set.
    Additionally, let $\mathcal{U}_z = h_{\x_0}(\mathcal{U}_x)$ and let $\rho_z(\cdot|\x_0)$ be the density of the pushforward distribution through $h_{\x_0}$, and assume $\mathcal{U}_z$ is star-shaped about the origin.
    The \emph{normal loss} over $\mathcal{U}_x$ for $h_{\x_0}^\theta$ is,
    \begin{align}
        \mathcal{L}_{\text{normal}}(\theta) = \mathbb{E}_{\z \sim \rho_z}\left[\e^\top g_z^\theta(\z)\e + \e^\top g_z^\theta(\z)^{-1}\e - 2\right]
        \label{eq:normal-loss}
    \end{align}
    where $\e = \z/\|\z\|_2$.
    The integrand is defined to be zero at $\z = \zero$.
    The normal loss satisfies $\mathcal{L}_{\text{normal}} \geq 0$ and $\mathcal{L}_{\text{normal}} = 0$ if and only if $h_{\x_0}^\theta$ is normal on $\mathcal{U}_x$.
\end{restatable}
See \Cref{app:normal-loss-proof} for a proof.
The integrand is the squared $g$-norm of the difference between the two sides of the normal chart identity, divided by the squared norm of $\z$ to remove the quadratic dependence on the distance from the origin.
We discuss the importance of this scaling in \Cref{sec:practical-considerations}.
Note that one can train with multiple $\x_0$ if amortization is desired because flow charts are conditional neural networks.
The main challenges in training with the normal loss stem from selecting the neighborhood $\mathcal U_x$ and distribution $\rho_x(\cdot\mid\x_0)$.
This selection is not trivial because \Cref{proposition:normal-loss} places geometric conditions on $\mathcal{U}_x$,
and the data might not come with a natural coupling to inform how to sample $\x \sim \rho_x(\cdot|\x_0)$.

\subsection{Practical considerations}
\label{sec:practical-considerations}
Suppose the user has access to a dataset of points $\mathcal{D} = \{\x^{(i)}\}_{i=1:N}$ and selects a subset $\mathcal{D}_0 \subseteq \mathcal{D}$ containing base points $\x_0 \in \mathcal{D}_0$.
The main practical challenge in training with \Cref{proposition:normal-loss} is determining how to sample $\z$ given $\x_0$, or equivalently by a change of variables $\x$ given $\x_0$, to estimate the normal loss.
To satisfy the assumptions of \Cref{proposition:normal-loss}, this sampler should be restricted to a neighborhood of $\x_0$ on which there exists a normal solution, and have positive density in this region.
Let us first consider how to satisfy the latter, as it can be satisfied independently of the former.
Given a procedure for sampling pairs of points $(\x_0,\x)$, we propose to draw coordinates to estimate the normal loss by sampling on the chord between $\zero = h_{\x_0}^\theta(\x_0)$ and $\z = h_{\x_0}^\theta(\x)$, and adding a small amount of Gaussian noise.
This amounts to selecting $\rho_z$ to be the distribution of $N(\cdot \mid t\,h^\theta_{\x_0}(\x), \gamma^2 I_n)$, where $t \sim \mathrm{U}(0,1)$.
In our experiments, we select $\gamma$ to be small relative to the chord length (see \Cref{sec:amortized-training}).
This sampler provides coverage along the endpoint chords and around the origin.

\begin{figure*}[t]
    \centering
    \includegraphics[width=\textwidth]{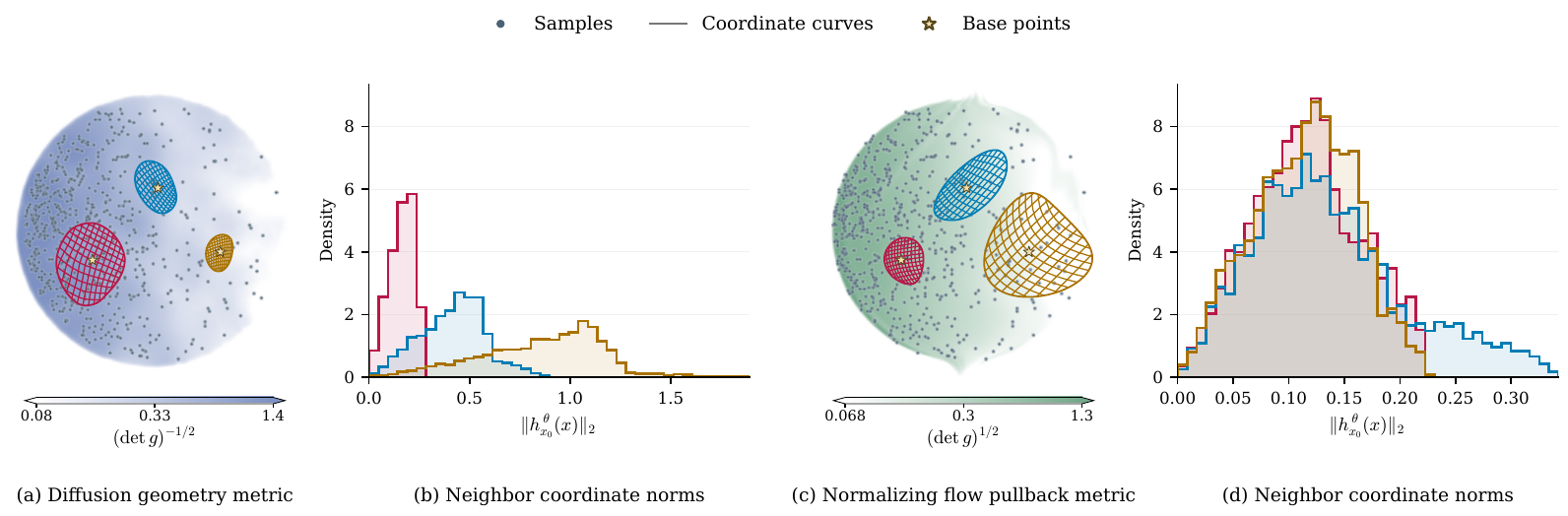}
    \caption{
    The scale of coordinates seen by a flow chart during amortized training depends jointly on how endpoint pairs are selected and the underlying geometry.
    (a) shows a diffusion geometry and (c) the pullback metric through a normalizing flow with a uniform prior, overlaid with the coordinate curves of learned flow charts contained in fixed radius coordinate balls.
    (b) and (d) show the distributions of coordinate norms for the nearest neighbors of each base point.
    The coordinates of endpoints selected with nearest neighbors have a scale that varies substantially across base points in (b), but is more comparable in (d) because in (a), coordinate balls tend to expand in the data space where observations are concentrated.
    }
    \label{fig:four-diffusion-charts}
\end{figure*}

In theory, every $\x$ should be restricted to a neighborhood of $\x_0$ that satisfies the assumptions of \Cref{proposition:normal-loss} so that the minimizer of the normal loss is normal.
On simple manifolds, like the sphere or torus, this neighborhood is known a priori, but in general it is unknown and determined by whether or not the exponential is a diffeomorphism on the region.
In this general setting, we propose to draw $\x$ from $\mathcal{D}$ using a heuristic that depends on $\x_0$, like Euclidean $k$ nearest neighbors.
This choice is motivated by our overarching goal of making flow charts a useful model for estimating solutions to geodesic queries by generalization from the pairs seen at training time.
However, decoupling the geometry from the training pairs can cause numerical issues that stem from the relative scale of coordinates seen by the model.
\Cref{fig:four-diffusion-charts} illustrates this.

The figure shows two geometries built from data with coordinate curves of learned flow charts, contained within a fixed radius coordinate ball that approximate geodesic balls.
Panel (a) shows a diffusion geometry \cite{jones2026computing}, whose metric is designed so that paths through regions supported by data are shorter than paths away from data \citep{hauberg2012geometric,arvanitidis2016locally,kapusniak2024metric,gaga}.
Panel (c) shows the pullback metric through a normalizing flow with a uniform prior that is fit to the data \citep{dombrowski2021diffeomorphic,diepeveen2024pulling}, whose Riemannian volume density approximates the probability density of data.
The colored background in (c) shows the intensity of the volume density in green, and the inverse volume density in blue for (a).
We can see that the size of the coordinate balls and the Riemannian volume density on each geometry relate to the density of data in opposite ways.
This behavior has practical implications for the scale of $\z$ that the model sees during training for a fixed endpoint sampling strategy.
For example, to generate this figure we sample from the Euclidean $k$ nearest neighbors to select $\x$ given $\x_0$.
(b) and (d) show the distribution of the norms of the neighbor coordinates for each base point.
There is more variation across the relative scales in (b) than (d), meaning that the model and optimization procedure must be robust to this variation.
This was our motivation for dividing the normal loss integrand by the squared norm of $\z$, which removes its explicit quadratic dependence on coordinate scale.
In early experimentation, we found that it was difficult to scale our approach to geometries over images and suspect that the variation in coordinate scales contributed to these difficulties.
We leave the investigation of other endpoint samplers for future work.

\section{Experiments}
\label{sec:experiments}
Our experiments test the accuracy of charts trained at fixed base points, their performance after pretraining over base points, and their use as coordinates for a numerical solver.
The experimental details are given in \Cref{app:protocol} and the full results for all tested settings are given in \Cref{app:complete-results}.

\paragraph{Geometries}
We evaluate on six families of manifolds with closed-form exponential and logarithmic maps: spheres, flat tori, hyperbolic spaces, manifolds of symmetric positive-definite matrices, Grassmann manifolds, and special orthogonal groups.
Each manifold is used in a low dimensional version, with $n\in\{2,3,4\}$, and a high dimensional version, with $n\in\{28,30\}$, and is observed through a nonlinear change of coordinates so that the observed coordinates are not already normal.
We also use three two-dimensional point clouds, a disk, two moons and nested circles, with the RBF metric of \citet{kapusniak2024metric}, and a diffusion geometry built from alanine dipeptide data ($n=30$), learned with Riemannian metric matching \citep{bamberger2026metricmatching}.
For these four data geometries, the reference logmap is computed with a GP solver \citep{arvanitidis2019fast} in double precision.
Each geometry has at most $128$ test base points with $1{,}024$ queries per base point on the closed form manifolds and $512$ on the data geometries, before discarding queries without a valid reference solution.
For each base point $\x_0$, let $\mathcal J(\x_0)$ denote the indices of the remaining queries and let $\mathcal I(\x_0)\subseteq\mathcal J(\x_0)$ denote the subset for which a method successfully produces a solution.

\paragraph{Flow chart architecture}
The flow chart we use in these experiments is a conditional neural spline flow \citep{durkan2019neural} with $L$ rational quadratic spline coupling layers, each with $8$ bins on $[-10, 10]$ in both low and high dimensions.
Before each coupling layer, we place a parametric elementwise scale and shift that is conditioned on the base point encoder and a dense invertible linear map (with an LU parameterization).
Each coupling layer parameterizes its spline using a FiLM conditioned \citep{perez2018film} residual network with three residual blocks.
The conditioning is computed with a separate residual network that encodes the base point into a $64$-dimensional feature vector for all coupling layers.
The network widths and other architectural choices are given in \Cref{app:models}.

We optimize the normal loss with the Adam optimizer \citep{kingma2014adam} for $2{,}000{,}000$ steps with a learning rate of $3 \times 10^{-4}$ with a linear warmup over the first $100$ steps and linear decay to $0$ for the last half of training, and a batch size of $512$.
When training over many base points, we found that it was beneficial to include multiple endpoints per base point and used $16$ base points with $32$ endpoints each.
We clip gradients to unit norm and maintain an exponential moving average of the parameters with decay $0.999$.
We evaluated on held-out validation pairs every $500$ steps and retain the checkpoint with the lowest validation loss.
The smaller four-layer charts used for the MFM comparison on data geometries used a constant learning rate after warmup and stopped early after $50{,}000$ steps without validation improvement.

\subsection{Evaluation}

\paragraph{Logarithmic map}
\label{sec:logmap}
The baselines we consider are all based on geodesic interpolation.
The first is metric flow matching (MFM) \citep{kapusniak2024metric}, which uses a neural network to parameterize a curve that is guaranteed to start at $\x_0$ and end at $\x$.
This network uses the same gated residual architecture as the flow chart's base point encoder.
Since flow charts use an expensive, invertible neural network architecture compared to MFM's unrestricted architecture, we vary the number of blocks and hidden sizes of the MFM to construct two baselines.
The first has a similar parameter count of a four-layer flow chart (``parameter matched''), and the second has a comparable compute time (``time matched'').
The MFM baselines were trained by minimizing the energy for $2{,}000{,}000$ steps, using the same optimization procedure as the flow charts.
The other two are the Gaussian process solver (GP) of \citet{arvanitidis2019fast} and GEORCE \citep{rygaard2026georce}, which solve for the geodesic between the endpoints by refining a curve parameterized by $K\in\{4,8,16,32\}$ intermediate accelerations and points respectively.
MFM serves as a fast, approximate baseline that can benefit from pretraining, while GP and GEORCE are numerical references that must be re-solved for each new pair of endpoints.

We train flow charts using the normal loss, and MFM by minimizing the energy of its curves.
In the non-amortized setting, we train a separate chart with $L=32$ coupling layers at each of four test base points, using training queries separate from its test queries.
For the runtime experiments, we separately vary chart capacity over $L\in\{4,8,16,32\}$ using the wider coupling networks described in \Cref{app:models}.

\begin{figure}[t]
\centering
\includegraphics[width=\columnwidth]{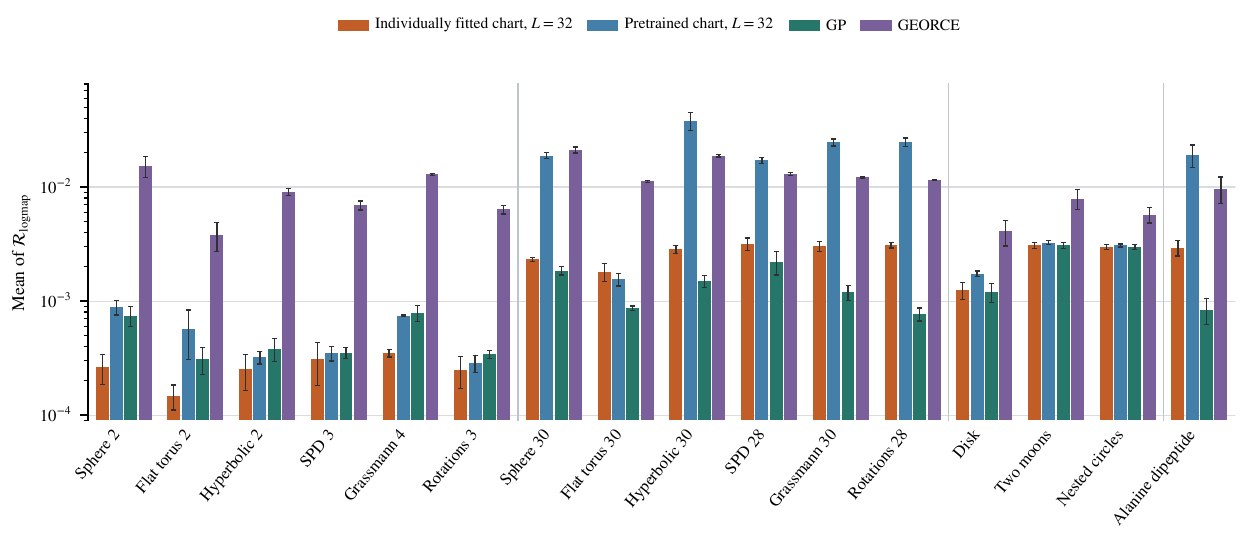}
\caption{Mean logmap normalized RMSE at four fixed base points with error bars indicating one standard error.
All methods were evaluated on the same four randomly selected test base points, and use their most favorable settings.
The charts use $L=32$ coupling layers, the GP uses $K=16$ intermediate accelerations, and GEORCE uses $K=32$ intermediate knot points.
All methods achieve low error, and the individually fitted flow chart achieves error comparable to the GP and GEORCE baselines, with mean error below $0.0032$ on every geometry.
The pretrained chart has higher error than the individually fitted chart on six of the seven high-dimensional geometries.
}
\label{fig:accuracy}
\end{figure}

\begin{figure}[t]
\centering
\includegraphics[width=\columnwidth]{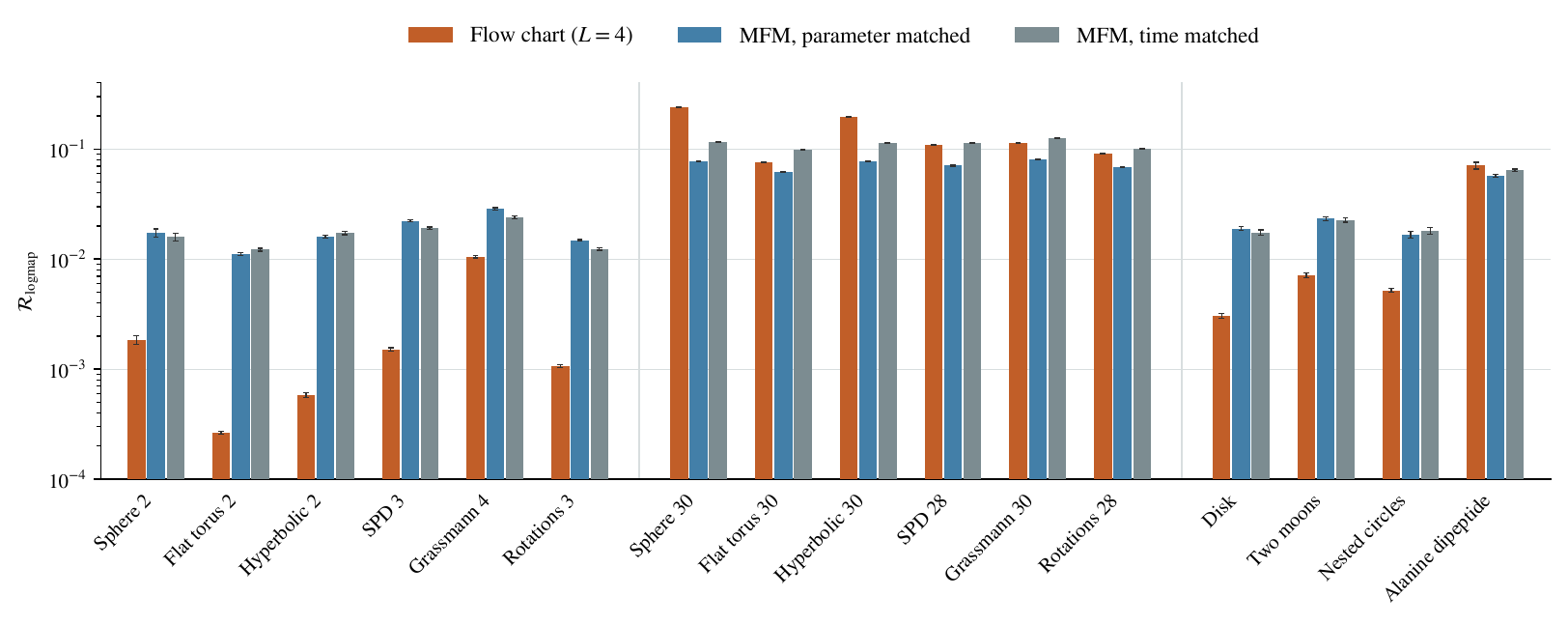}
\caption{
Amortized logmap accuracy relative to MFM.
Each bar denotes the mean $\Rlog$ of the $L=4$ flow chart or a parameter matched or time matched MFM at the same held out base points.
Error bars indicate one standard error.
In most settings the flow chart outperforms the baselines, but the parameter matched method performs best in the high dimensional synthetic geometries.
}
\label{fig:accuracy-amortized}
\end{figure}

Let $U_x^{(i)}$ be a method's logmap solution and $V_x^{(i)}$ the reference solution for query $i$ at $\x_0$.
We measure accuracy using the root mean squared error over the successful queries, $\Elog(\x_0)$, and normalize it by the RMS norm of the reference logmaps to account for differences in scale across geometries to obtain the normalized RMSE, $\Rlog(\x_0)$.
Dropping the dependence on $\x_0$ for brevity, these quantities are defined as:
\[
\Elog
=\sqrt{\frac{1}{|\mathcal I|}\sum_{i\in\mathcal I}
\|U_x^{(i)}-V_x^{(i)}\|_g^2},\;
r_{\mathrm{ref}}
=\sqrt{\frac{1}{|\mathcal J|}\sum_{i\in\mathcal J}
\|V_x^{(i)}\|_g^2},\;
\Rlog=\frac{\Elog}{r_{\mathrm{ref}}}
\]
For each geometry, we report the mean of $\Rlog(\x_0)$ across base points and its standard error.
We also report the wall clock time, $\tau$, which is measured in microseconds per query on one NVIDIA A16 in single precision, with compilation and preprocessing excluded.

\Needspace{4\baselineskip}
\paragraph{Exponential map}
\label{sec:expmap}
We use the same geometries to evaluate the exponential map.
Let $(\x_0,V_x,\x_1)$ be a test triplet, where $V_x=\log_{\x_0}(\x_1)$ is known in closed form or estimated by the reference GP solver.
Given $\x_0$ and $V_x$, each method is tasked with recovering the endpoint $\x_1$.
The baselines simulate the geodesic with initial position $\x_0$ and velocity $V_x$ by integrating Hamilton's equations with Euler, Heun, Tsit5 \citep{tsitouras2011runge} and Dopri8 \citep{prince1981high} with ranges of steps selected to cap the total number of cometric evaluations at $129$.
We use the implementation of these solvers provided by the Diffrax library  \citep{kidger2021on}.

Let $\gamma_1^{(i)}$ be a method's endpoint solution for query $i$ at $\x_0$.
We measure accuracy using the root mean squared error over the successful queries, $\Eexp(\x_0)$, and normalize it by the RMS Euclidean distance from the base point to the reference endpoints to obtain the normalized RMSE, $\Rexp(\x_0)$:
\[
\Eexp
=\sqrt{\frac{1}{|\mathcal I|}\sum_{i\in\mathcal I}
\|\gamma_1^{(i)}-\x_1^{(i)}\|_2^2},\;
r_{\mathrm{chord}}
=\sqrt{\frac{1}{|\mathcal J|}\sum_{i\in\mathcal J}
\|\x_1^{(i)}-\x_0\|_2^2},\;
\Rexp
=\frac{\Eexp}{r_{\mathrm{chord}}}
\]
For this experiment, a query is successful if the method returns a finite endpoint.
The reference scale uses all queries in $\mathcal J(\x_0)$ and is shared across methods.
For each geometry, we report the mean of $\Rexp(\x_0)$ across base points and its standard error, both multiplied by $100$, and the mean wall clock time per query, $\tau$.

\paragraph{Runtime comparisons}
We also compare the runtimes of the methods needed to achieve a common accuracy.
For the logarithmic map, we compare pretrained flow charts with GP and GEORCE, and for the exponential map we compare them with the ODE solvers.
All methods are timed on the same NVIDIA A16 gpu in the same way, and we sweep over various batch sizes to solve multiple GP and GEORCE queries at once, and report the time for the fastest batch size.
For each geometry and comparison, we find the lowest error threshold that both methods can reach among their tested settings, then compare their fastest settings that meet this threshold.
See \Cref{app:matching} for details.

\begin{figure*}[!t]
\centering
\includegraphics[width=0.98\textwidth,trim=0 7bp 0 5bp,clip]{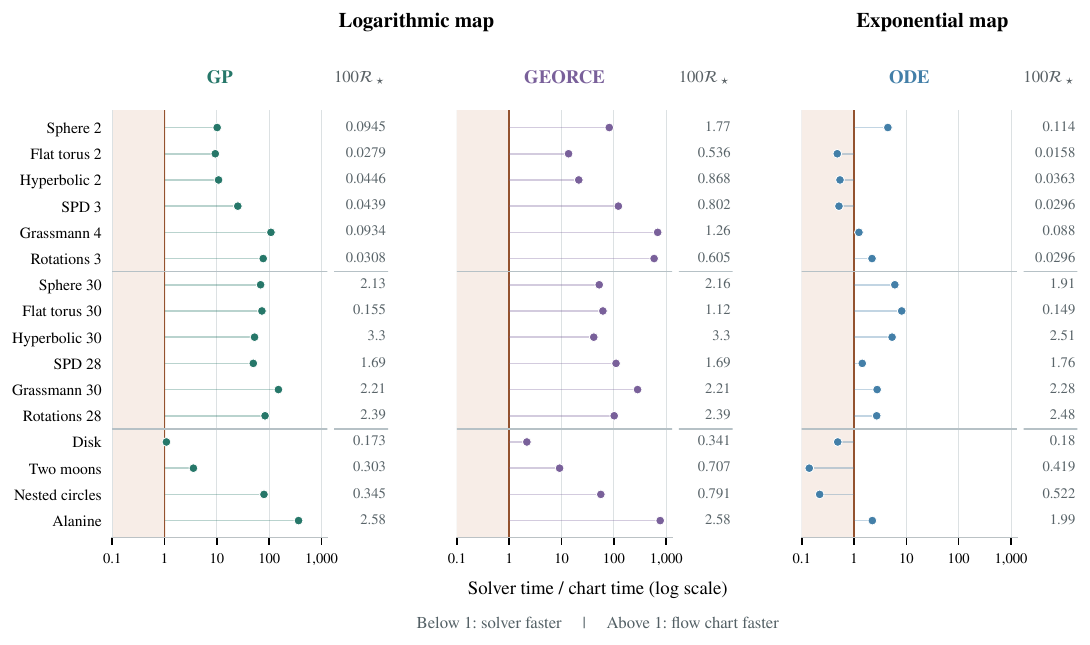}
\caption{Comparison of time needed to achieve similar error between flow charts and numerical methods.
For each comparison, $\mathcal{R}_\star$ is the lowest error level that both methods can reach among the settings we tested.
Each dot is the numerical method's solve time divided by the chart's solve time, using each method's fastest setting with error at most $\mathcal{R}_\star$.
Values above one favor the chart.
Flow charts are faster than the numerical baselines, except for exponential-map computation on the low-dimensional data geometries, whose metrics are fast to evaluate.
}
\label{fig:speed}
\end{figure*}

\subsection{Results}
\Cref{fig:accuracy} shows that a flow chart trained at a fixed base point can recover the logarithmic map to a degree comparable to the baselines.
Over $4$ base points and $1024$ and $512$ test queries for closed form and data-driven geometries, respectively, the flow chart achieved an average error below $0.0032$, and the pretrained flow chart outperformed GEORCE on $11$ of the $16$ geometries.
The chart is more accurate than GEORCE on all sixteen geometries, but slightly less accurate than GP on the high dimensional manifolds.
These results empirically validate the correctness of training flow charts for the normal loss.

\Cref{fig:accuracy-amortized} compares the performance of a flow chart with $L=4$ against the MFM baselines.
The flow chart is more accurate than both MFMs on all nine geometries with $n\leq4$.
On the seven geometries with $n\geq28$, its error is within a factor of $2.1$ of time-matched MFM, and is lower on four of them.
This could be because the flow chart has a restrictive neural network architecture, so it needs more parameters than MFM's unconstrained architecture to achieve the same accuracy.

For the logarithmic map, \Cref{fig:speed} shows that the amortized flow chart is up to $364\times$ faster than GP and $766\times$ faster than GEORCE, with median speedups of $60\times$ and $72\times$ across the sixteen geometries.
On the high dimensional closed form manifolds, the speedups are $50$--$150\times$ against GP and $41$--$283\times$ against GEORCE.
On alanine dipeptide, both comparisons use a threshold of $2.58$ on the mean of $100\times\Rlog$, which the chart reaches $364\times$ faster than GP and $766\times$ faster than GEORCE.

For the exponential-map comparison in \Cref{fig:speed}, we vary the chart size and the number of steps used by each ODE solver.
For each geometry, we select the fastest chart and the fastest setting across all four ODE solvers whose mean $\Rexp$ is at most $\mathcal{R}_\star$.
Each point shows the ODE runtime divided by the chart runtime, so values above one favor the chart.
For the exponential map on the seven geometries with $n\geq28$, \Cref{fig:speed} shows that flow charts are $1.4$--$8.2\times$ faster than the ODE solvers to achieve comparable error, and on the two dimensional point clouds, the ODE solvers are faster than the charts.
Here, the fastest solver takes $0.14$--$0.49$ times the runtime of the fastest chart.
\Cref{app:expmap-results} reports the exponential-map errors and the comparison at a fixed runtime budget.

\subsection{Flow chart preconditioner}
\label{sec:preconditioner}
Finally, we show that pretrained flow charts can be used to refine the logmap solution of GEORCE.
Flow charts learn a coordinate space where geodesics are approximated by radial lines through the origin.
We hypothesize that this behavior can make the task of solving for geodesics simpler to boost the performance of numerical methods.
We test this by comparing the performance of GEORCE in the data space and in the coordinate space.
Let $\x_0$ be a test base point and $\x$ one of its queries.
We encode the query with the pretrained flow chart, and use GEORCE to solve for the geodesic from $\zero$ to $h_{\x_0}^{\theta}(\x)$ under the coordinate representation of the metric, $g_z^\theta$.
We use the same test base points and queries, knot counts, stopping tolerance, iteration cap and precision as for GEORCE in the data coordinates.
The pretrained flow chart has $L=4$ coupling layers and we run GEORCE at knot counts $K\in\{4,8,16,32\}$.
\Cref{tab:preconditioner} shows that solving GEORCE in the learned coordinate space consistently improves its performance.

\begin{table}[H]
\centering
\caption{
GEORCE in data coordinates and in the coordinates of a pretrained flow chart.
Columns vary GEORCE's knot count $K$.
Each cell reports the mean of $100\times \Rlog$ across base points in data coordinates followed by an arrow and the value in chart coordinates.
The bold numbers above the arrows show the ratio between the two numbers.
Solving in chart coordinates consistently lowers error.
}
\label{tab:preconditioner}
\includegraphics[width=\textwidth]{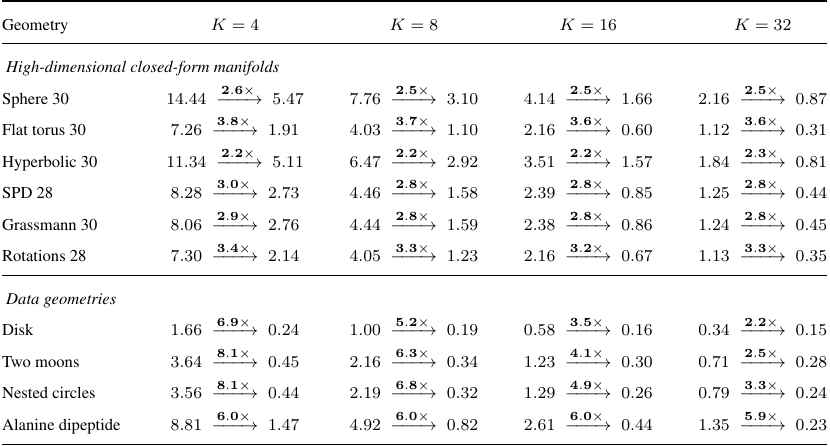}
\end{table}

\section{Conclusion}
\label{sec:conclusion}
We introduced flow charts, which are conditional invertible neural networks that learn normal coordinates on a given Riemannian manifold.
Our training objective is non-negative and vanishes precisely at normal coordinate charts, under suitable domain assumptions, enabling direct approximations of both the exponential and logarithmic maps.
Across synthetic and data-derived geometries, flow charts approximate the maps on par with numerical methods, while being substantially faster when pretrained on data, and they are comparable in performance to MFM for computing the logarithmic map.
The main limitation of our method is that it requires heuristics to select training points and that our method can only answer local geodesic queries.
Future work could address these limitations by investigating different training strategies, using flow charts to define their own geometry, and by re-centering charts along longer geodesics to answer queries beyond a single chart's domain.

\subsection*{AI use statement}

The research ideas and the design of the method and loss functions originated with the authors.
We used ChatGPT and Claude to assist with designing and carrying out experiments, writing and debugging code, running jobs, analyzing results, and producing figures.
These tools also helped optimize computational performance, including the efficient computation of geodesic accelerations described in the appendix.
We used them to suggest draft proofs, check derivations, and proofread the manuscript.
We personally checked the experimental setup, results, and proofs, and take responsibility for the final content of this work, including all AI-assisted text, claims, code, and figures.

\bibliographystyle{plainnat}
\bibliography{references}

\clearpage
\appendix
\begin{center}
{\Large\bfseries Supplementary Materials}
\end{center}
\medskip

\section{Proofs}
\label{app:proofs}

\begin{lemma}[Contracted Christoffel symbols]
    \label{lem:contracted-christoffel}
    Let $(\M, g)$ be a $n$-dimensional Riemannian manifold, $x\colon \M \to \M_x$ be a coordinate system, $\x \in \M_x$ be a point, $V = V_x^i \frac{\partial}{\partial x^i} \in T_\x \M_x$ be a tangent vector and let $\Gamma_x(\x)^k_{ij}$ denote the Christoffel symbols of its Levi-Civita connection at $\x \in \M_x$.
    Let $\Gamma_x(\x)(V_x, V_x) \in \R^n$ denote the contraction $\Gamma_x(\x)(V_x, V_x)^k = \Gamma_x(\x)^k_{ij}\, V^i V^j$.
    Then
    \begin{align}
        \Gamma_x(\x)(V_x, V_x) = g_x(\x)^{-1}\left[D_V (g_x(\x)V_x) - \frac{1}{2}\nabla_x \left(V_x^\top g_x(\x)V_x
        \right)\right]
        \label{eq:contracted-christoffel}
    \end{align}
    where $V$ is constant with respect to the directional derivative $D_V$ and the Euclidean gradient $\nabla_x$.
\end{lemma}
\begin{proof}
    The Christoffel symbols of the Levi-Civita connection are given by
    \begin{align}
        \Gamma_x(\x)^k_{ij}
        =
        \frac{1}{2} g_x(\x)^{km}
        \left(
        \frac{\partial g_{x,jm}(\x)}{\partial x^i}
        +
        \frac{\partial g_{x,im}(\x)}{\partial x^j}
        -
        \frac{\partial g_{x,ij}(\x)}{\partial x^m}
        \right)
    \end{align}
    Contracting with $V_x^i V_x^j$ makes the first two terms equal because of the metric is symmetric.
    Therefore
    \begin{align}
        \Gamma_x(\x)(V_x, V_x)^k
        =
        g_x(\x)^{km}
        \left(
        \frac{\partial g_{x,im}(\x)}{\partial x^j}\, V_x^i V_x^j
        -
        \frac{1}{2}
        \frac{\partial g_{x,ij}(\x)}{\partial x^m}\, V_x^i V_x^j
        \right)
    \end{align}
    Converting this expression to matrix notation, and passing the derivatives through $V_x$ yields the result.
\end{proof}

\gausslemma*
\begin{proof}
    \label{app:normal-chart-identity-proof}
    First we will show that when $h_{\x_0}$ is normal, we have that $g_z(\z)\z = \z$ using Gauss's lemma.
    Since $h_{\x_0}$ is normal, we have that $f_{\x_0}(\z) = \exp_{\x_0}(\mathbf{J}_x \z)$, where $\mathbf{J}_x = Df_{\x_0}(\zero)$.
    By the chain rule, we have that $Df_{\x_0}(\z) = D\exp_{\x_0}(\mathbf{J}_x \z)\mathbf{J}_x$.
    Suppose $V_x = \mathbf{J}_x \z$ and $W_x = \mathbf{J}_x w$ for some $w \in \R^n$.
    Recall Gauss's lemma \citep[Chapter 3, Lemma 3.5]{docarmo1992riemannian}, which says that for $\x_0 \in \M_x$ and $V_x \in \R^n$ in the domain of $\exp_{\x_0}$, we have that for every $W_x \in \R^n$, the following expression holds
    \begin{align}
        \left(D\exp_{\x_0}(V_x)\,V_x\right)^\top
        g_x(\x_1)
        \left(D\exp_{\x_0}(V_x)\,W_x\right)
        =
        V_x^\top g_x(\x_0)\,W_x,
        \qquad
        \x_1 = \exp_{\x_0}(V_x)
    \end{align}
    where $D\exp_{\x_0}(V_x)$ denotes the Jacobian matrix of the exponential map evaluated at $V_x$.
    Then for $\x = f_{\x_0}(\z)$, the LHS can be rewritten as follows
    \begin{align}
        V_x^\top D\exp_{\x_0}(V_x)^\top g_x(\x)D\exp_{\x_0}(V_x) W_x &= \z^\top \mathbf{J}_x^\top D\exp_{\x_0}(V_x)^\top g_x(\x)D\exp_{\x_0}(V_x) \mathbf{J}_x w \\
        &= \z^\top Df_{\x_0}(\z)^\top g_x(\x)Df_{\x_0}(\z) w \\
        &= \z^\top g_z(\z)w
    \end{align}
    And the RHS can be written as
    \begin{align}
        V_x^\top g_x(\x_0)W_x = \z^\top \underbrace{\mathbf{J}_x^\top g_x(\x_0)\mathbf{J}_x}_{\mathbf{J}_x\text{ is $g$-orthogonal}} w = \z^\top w
    \end{align}
    Since $w$ is arbitrary, we are left with the identity, $g_z(\z)\z = \z$.

    Next, we will prove that the identity $g_z(\z)\z = \z$ implies that $h_{\x_0}$ is normal.
    The proof follows in three steps.
    \begin{enumerate}
        \item[{\circlenum{1}}] The identity near $\zero$ forces $g_z(\zero) = I_n$, and so $\mathbf{J}_x := Df_{\x_0}(\zero)$ is $g$-orthogonal.
        \item[{\circlenum{2}}] The identity on $\Z_{\x_0}$ forces $\Gamma_z(\z)(\z, \z) = \zero$, and so $f_{\x_0}(\z_t)$ is a geodesic.
        \item[{\circlenum{3}}] {\circlenum{1}} + {\circlenum{2}} $\implies h_{\x_0}(\x) = \mathbf{J}_x^{-1} \log_{\x_0}(\x)$
    \end{enumerate}

    \paragraph{\circlenum{1}}

    First we show that the identity implies that $\mathbf{J}_x$ is $g$-orthogonal.
    Let $\e \in \R^n$ be a unit length vector, $\|\e\|_2 = 1$, and let $s > 0$ be small enough for the vector $\z = s\e$ to be contained in $\Z_{\x_0}$.
    The identity then reads as $g_z(s\e)s\e = s\e$.
    Dividing by $s$, and letting $s \to 0$, we get that $g_z(\zero)\e = \e$.
    Since $\e$ is arbitrary, $g_z(\zero) = I_n$, and so $Df_{\x_0}(\zero)^\top g_x(\x_0)Df_{\x_0}(\zero) = I_n$, which states that $\mathbf{J}_x := Df_{\x_0}(\zero)$ is $g$-orthogonal.

    \paragraph{\circlenum{2}}

    Next we will show that the identity implies that $f_{\x_0}(\z_t)$ is a geodesic.
    Writing the Christoffel contraction from \cref{lem:contracted-christoffel} at $V_z = \z$ gives
    \begin{align}
        \Gamma_z(\z)(\z, \z) = g_z(\z)^{-1}\left[D_{\z}\left(g_z(\z)\z'\right) - \frac{1}{2}\nabla_z\left({\z'}^\top g_z(\z)\z'\right)\right], \quad \text{where }\z' = \mathrm{stop\_gradient}(\z)
        \label{app:eq-christoffel-radial}
    \end{align}
    Under the identity, each of the two terms in the bracket equals zero.

    Let $\delta(\z) = g_z(\z)\z - \z$.
    For the first term, consider the directional derivative $D_\z \delta(\z)$.
    Expanding the expression and rearranging terms yields
    \begin{align}
        D_{\z}(g_z(\z)\z') =  D\delta(\z)\z - \delta(\z)
    \end{align}
    where $D\delta(\z)$ is the Jacobian matrix of $\delta$ at $\z$.

    For the second term, consider $\nabla_z(\z^\top \delta(\z))$.
    This can be expanded and rearranged to yield
    \begin{align}
        \nabla_z\left({\z'}^\top g_z(\z)\z'\right) = D\delta(\z)^\top \z - \delta(\z)
    \end{align}

    Plugging both terms in gives the expression for the Christoffel contraction in terms of $\delta$:
    \begin{align}
        \Gamma_z(\z)(\z, \z) = g_z(\z)^{-1}\left[D\delta(\z)\z - \frac{1}{2}D\delta(\z)^\top \z - \frac{1}{2}\delta(\z)\right]
        \label{app:eq-christoffel-radial-simplified}
    \end{align}
    Since $\Z_{\x_0}$ is open, and $\delta(\z) = 0$ for all $\z \in \Z_{\x_0}$, $D\delta(\z) = 0$.
    So we have $\Gamma_z(\z)(\z, \z) = \zero$.

    Now suppose that $\z \in \Z_{\x_0}$, and let $\z_t = t\z$ for $t \in (0,1]$ be a curve.
    Note that $\z_t \in \Z_{\x_0}$ by star shapedness.
    We will show that $\z_t$ is a geodesic on $(\M_z, g_z)$.
    Let $V_t = \dot{z}_t^i \frac{\partial}{\partial z^i}$ be the velocity vector of the curve, and note that $\dot{z}_t = \z$ and $\ddot{z}_t = 0$.
    By definition of a geodesic, $\z_t$ is a geodesic if and only if it has zero covariant acceleration, $\nabla_{V_t}V_t$.
    Expanding the term, we find that its components are $0$:
    \begin{align}
        \nabla_{V_t}V_t &= \left(\dot{z}_t^i \frac{\partial \dot{z}_t^k}{\partial z^i} + \Gamma_{z,ij}^k \dot{z}_t^i \dot{z}_t^j\right)\frac{\partial}{\partial z^k} \\
        &= \left(\underbrace{\frac{d\dot{z}_t^k}{dt}}_{\ddot{z}_t=0} + \Gamma_z(\z_t)(\dot{z}_t, \dot{z}_t)^k\right)\frac{\partial}{\partial z^k} \\
        &= \left(\frac{1}{t^2}{\underbrace{\Gamma_z(\z_t)(\z_t, \z_t)}_{=0}}^k\right)\frac{\partial}{\partial z^k} \\
        &= 0
    \end{align}
    The second to last step follows from bilinearity of $\Gamma_z(\z)(\cdot, \cdot)$.
    The covariant acceleration of $\z_t$ is therefore $0$ and $\z_t$ is a geodesic.
    By coordinate invariance, $\x_t = f_{\x_0}(\z_t)$ is also a geodesic.

    \paragraph{\circlenum{3}}
    Recall the initial value form of geodesics, $\x_t = \exp_{\x_0}(tV_x)$, where $V_x$ is an initial velocity and $\x_0$ is the starting point.
    The initial velocity of our geodesic is $\frac{d}{dt}|_{t=0}f_{\x_0}(\z_t) = Df_{\x_0}(\zero)\z = \mathbf{J}_x \z$.
    Equating the expressions gives
    \begin{align}
        \label{app:eq-geodesic-exp-map}
        f_{\x_0}(\z_t) = \exp_{\x_0}(t\mathbf{J}_x \z), \quad t \in (0,1]
    \end{align}
    If we can invert the expression at $t=1$, then we are done, but before doing so we must verify that $\mathbf{J}_x \z$ lies in a domain on which $\exp_{\x_0}$ is invertible.

    Recall that the geodesic ball $\X_{\x_0}$ is defined as $\X_{\x_0} = \exp_{\x_0}(B)$, where $B = \{V_x \in \R^n \colon V_x^\top g_x(\x_0)V_x < \varepsilon^2\} \subset T_{\x_0}\M_x$.
    The proof is done if we show that $\mathbf{J}_x \z \in B$ so that \Cref{app:eq-geodesic-exp-map} can be inverted.
    By \circlenum{1}, $\mathbf{J}_x$ is $g$-orthogonal and so $(\mathbf{J}_x\z)^\top g_x(\x_0)(\mathbf{J}_x\z) = \z^\top\z$, and so an equivalent condition to $\mathbf{J}_x \z \in B$ is $\|\z\|_2 < \varepsilon$.

    Consider the curve $X_t = \log_{\x_0}\left(f_{\x_0}(\z_t)\right)$, which is continuous on $[0,1]$ and takes values in $B$.
    Whenever $t\|\z\|_2 < \varepsilon$, the vectors $X_t$ and $t\mathbf{J}_x \z$ both lie in $B$ and satisfy $\exp_{\x_0}(X_t) = f_{\x_0}(\z_t) = \exp_{\x_0}(t\mathbf{J}_x \z)$, so the injectivity of $\exp_{\x_0}$ on $B$ gives $X_t = t\mathbf{J}_x \z$.

    If $\|\z\|_2 \geq \varepsilon$, this would hold for every $t < t^* = \varepsilon / \|\z\|_2 \leq 1$, and continuity would give $X_{t^*} = t^*\mathbf{J}_x\z$, whose norm satisfies $X_{t^*}^\top g_x(\x_0)\,X_{t^*} = \varepsilon^2$ and contradicts the strict inequality that defines $B$.
    Therefore $\|\z\|_2 < \varepsilon$, and so $\mathbf{J}_x\z \in B$.
    This means that $f_{\x_0}(\z) = \exp_{\x_0}(\mathbf{J}_x \z) = \x$ admits a unique inverse, giving us the result
    \begin{align}
        h_{\x_0}(\x) = \mathbf{J}_x^{-1}\log_{\x_0}(\x), \quad \x \in \mathcal{U}_x
    \end{align}

\end{proof}

\normalityerror*
\begin{proof}
    \label{app:normal-loss-proof}
    We drop the dependence on $\theta$ for brevity.
    Let $\delta(\z) = g_z(\z)\z - \z$ denote the amount by which the normal chart identity \Cref{eq:normal-chart-identity} fails at $\z \in \mathcal{U}_z$.

    $\delta(\z)$ contains the $z$ components of a covector, because $g_z(\z)\z$ is the covector $V^\flat_z = g_z(\z)V_z$ associated with the vector $V_z = \z$ and $\z = \nabla_z \frac{1}{2}\|\z\|_2^2$ is the gradient of a function, so $\z$ contains the components of the covector $d\frac{1}{2}\|\z\|_2^2$.
    Therefore $\delta(\z)$ is the difference of two covectors components.
    This means that its squared $g$-norm is defined by $\delta(\z)^\top g_z(\z)^{-1}\delta(\z)$.

    First we will show that the integrand of \Cref{eq:normal-loss} is this squared $g$-norm divided by $\|\z\|_2^2$.
    Expanding the square norm gives
    \begin{align}
        \delta(\z)^\top g_z(\z)^{-1}\delta(\z)
        &= \left(g_z(\z)\z - \z\right)^\top g_z(\z)^{-1}\left(g_z(\z)\z - \z\right) \\
        &= \z^\top g_z(\z)\z - 2\z^\top \z + \z^\top g_z(\z)^{-1}\z
        \label{app:eq-residual-norm}
    \end{align}
    Dividing by $\|\z\|_2^2$ and writing $\e = \z / \|\z\|_2$ gives $\e^\top g_z(\z)\e + \e^\top g_z(\z)^{-1}\e - 2$, which is the integrand of \Cref{eq:normal-loss} for $\z \neq \zero$.

    Since $g_z(\z)^{-1}$ is positive definite, the integrand is a squared norm divided by a positive number, so it is nonnegative for $\z \neq \zero$, and it is zero at $\z = \zero$ by definition.
    Its expectation under $\rho_z$ is therefore nonnegative, which gives $\mathcal{L}_{\text{normal}} \geq 0$.
    For $\z \neq \zero$, the integrand equals zero if and only if $\delta(\z) = \zero$.

    Next we will show that $\mathcal{L}_{\text{normal}} = 0$ if and only if $h_{\x_0}$ is normal on $\mathcal{U}_x$.
    The set $\mathcal{U}_x$ is open, contains $\x_0$, lies in the geodesic ball $\X_{\x_0}$, and has star shaped image $\mathcal{U}_z$, so by \Cref{prop:normal-chart-identity}, $h_{\x_0}$ is normal on $\mathcal{U}_x$ if and only if $\delta$ vanishes on $\mathcal{U}_z$.
    If $\delta$ vanishes on $\mathcal{U}_z$ then the integrand vanishes on the support of $\rho_z$, so $\mathcal{L}_{\text{normal}} = 0$.
    Conversely, if $\mathcal{L}_{\text{normal}} = 0$ then the integrand, which is nonnegative and continuous on $\mathcal{U}_z \setminus \{\zero\}$, has zero expectation under $\rho_z$, and $\rho_z$ is positive on $\mathcal{U}_z$ because $\rho_x(\cdot|\x_0)$ is positive on $\mathcal{U}_x$ and $f_{\x_0}$ is a diffeomorphism.
    A nonnegative continuous function with zero expectation under a positive density on an open set vanishes on that set, so $\delta(\z) = \zero$ for all $\z \in \mathcal{U}_z \setminus \{\zero\}$, and $\delta(\zero) = \zero$ holds trivially.
\end{proof}

\section{Datasets}
\label{app:datasets-and-rmm-geometries}
All of the datasets are split into a training, validation and test set using a $(70,10,20)$ split.
The size and dimensionality of each dataset are listed below.

\begin{center}
    \begin{tabular}{lrrrr}
        \toprule
        & Alanine dipeptide & Disk & Two moons & Nested circles \\
        \midrule
        $n$ & $30$ & $2$ & $2$ & $2$ \\
        Elements & $750{,}000$ & $500{,}000$ & $500{,}000$ & $500{,}000$ \\
        \bottomrule
    \end{tabular}
\end{center}

\paragraph{Alanine dipeptide}
The dataset consists of $750{,}000$ snapshots of a molecule over three $250$ ns trajectories \citep{nuske2017markov,wehmeyer2018timelagged}.
Each snapshot contains the three-dimensional positions of $10$ heavy atoms (the full molecule has 12 extra hydrogen atoms that are not measured) so that each data point is a vector in $\R^{30}$.
The three trajectories were processed to have the same centroid and orientation, which removes $6$ degrees of freedom.
We remove the time ordering by shuffling the data, and preprocess by centering and scaling the data by a single scalar so the coordinates have unit root mean square spread.

This dataset is highly structured.
As mentioned, the data is preprocessed to shift and orient each trajectory to the same view, so the data manifold is restricted to a $24$ dimensional submanifold.
Additionally, the mean distance between pairs of atoms across snapshots can be used to determine whether or not each pair is bonded.
This yields a graph structure over the atoms in $3$d space that characterizes the shape of the molecule.
Most of the shape is captured by two main backbone dihedral angles, $\phi$ and $\psi$, which represent most of the variation of the molecule.
As such, it is common to project the data points down to the Ramachandran plot \citep{ramachandran1963stereochemistry}, which is a $2$d torus whose axes are each of the dihedral angles.
The projection of data onto this plane organizes into clusters that, at a high level, represent the different shapes that the molecule transitions in and out of during a trajectory.
For this reason, the Ramachandran plane is often used during evaluation to interpret trajectories.

\paragraph{Synthetic datasets}
The three synthetic datasets are point clouds of $500{,}000$ points in $\R^2$, each generated with a fixed random seed.
The two moons and the nested circles are generated with the \texttt{make\_moons}, with a noise standard deviation of $0.06$, and \texttt{make\_circles}, with an inner radius of $0.5$ and a noise standard deviation of $0.04$, from Scikit-learn \citep{pedregosa2011scikit}.
We generate the disk data from a density proportional to $\exp(-1.8(x_1 + 1))$ on the unit disk using rejection sampling.
After the split, each dataset is centered and scaled by a single scalar so that the training set has unit root mean square spread, and the same transform is applied to the validation and test sets.

\paragraph{Licenses and data permissions}
The alanine trajectories are publicly provided by the Computational Molecular Biology Group at Freie Universit\"at Berlin through \texttt{mdshare} under CC BY 4.0, which permits reuse with attribution.\footnote{Dataset and provider's license statement: \url{https://markovmodel.github.io/mdshare/ALA2/} and \url{https://markovmodel.github.io/mdshare/}.}
The remaining datasets are generated synthetically; no human participants are involved.
The MFM reference implementation and Jaxtyping use the MIT license; JAX, Equinox, Optax, and Diffrax use Apache 2.0; NumPy, SciPy, and Scikit-learn use BSD 3-Clause; \texttt{tqdm} uses MPL 2.0 and MIT; and Matplotlib uses its PSF-based license.

\section{Geometries}
\label{app:experimental-geometries}
We test our method on synthetic metrics that admit closed form solutions, RBF metrics on three point clouds, and a metric learned from alanine dipeptide with Riemannian metric matching \citep{bamberger2026metricmatching}.

\subsection{Synthetic metrics}
\label{sec:synthetic-metrics}
We generate data on the six synthetic manifolds in \Cref{tab:synthetic-metrics}, whose exponential maps, logarithmic maps and geodesic distances are known in closed form.
Normal coordinates would already give a valid solution at their center, so we apply a nonlinear change of coordinates before presenting the data to our method.
For each manifold, we construct one fixed observed coordinate system $x$, then sample base points $\x_0$ at which we train the normal chart $h_{\x_0}^\theta$.
The observed coordinates are shared across base points.
See \Cref{fig:synthetic-coordinates} for a visual overview.

\begin{figure}[t]
    \centering
    \includegraphics[width=\textwidth]{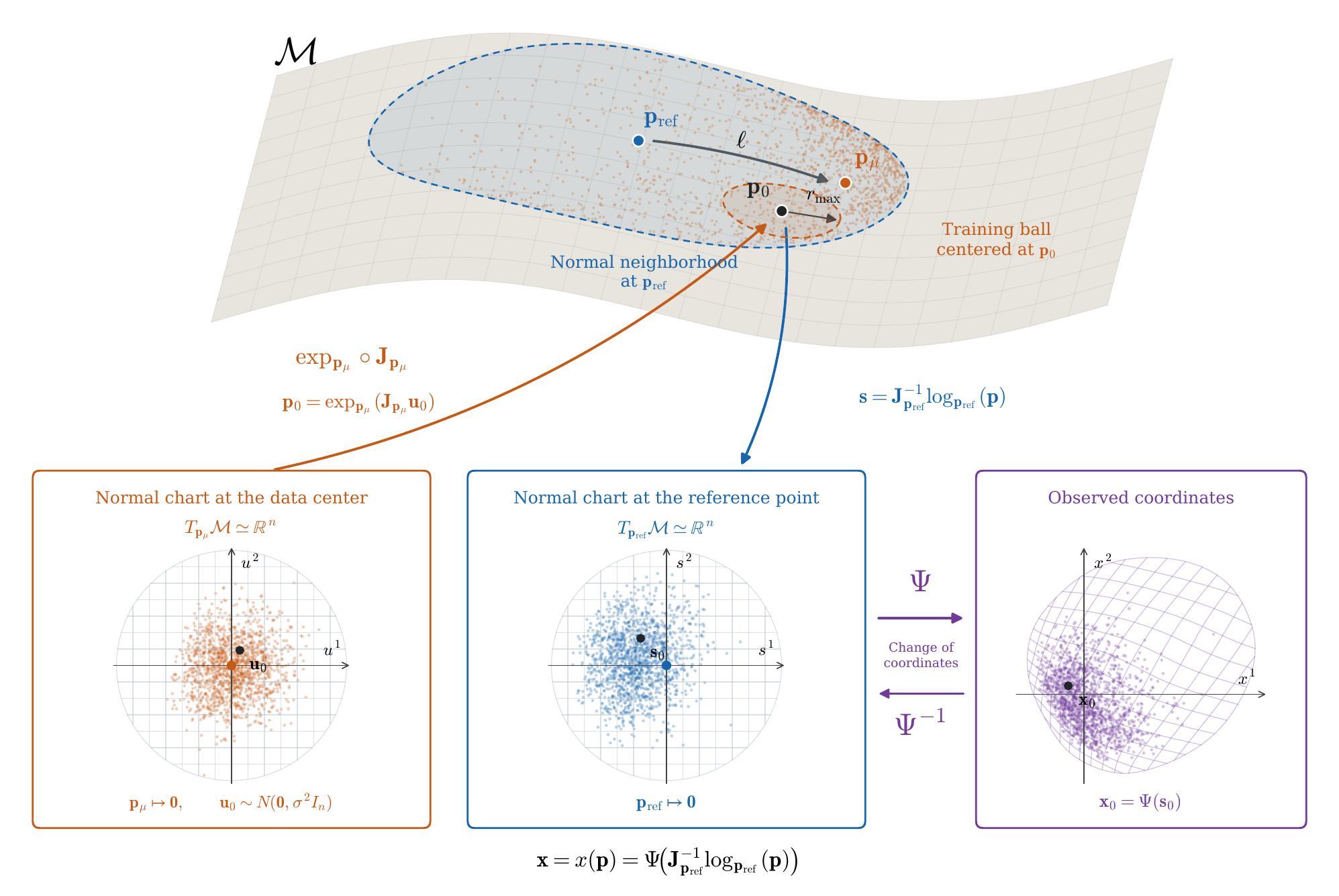}
    \caption{Synthetic data generation.
    We sample in normal coordinates at the data center $\p_\mu$ (orange), express the same points in normal coordinates at the reference point $\p_{\mathrm{ref}}$ (blue), and apply $\Psi$ to obtain observed coordinates (purple).
    The learned chart is centered at a sampled base point $\p_0$, whose observed coordinates are $\x_0=x(\p_0)$.}
    \label{fig:synthetic-coordinates}
\end{figure}

Let $n=\dim\M$ and let $\iota$ be the injectivity radius.
Let $\p\in\M$ for a point and $\mathbf{J}_{\p}\colon\R^n\to T_{\p}\M$ for a $g$-orthonormal frame at that point.

We center the normal coordinates at a fixed reference point $\p_{\mathrm{ref}}$ and sample base points around a data center $\p_\mu$.
We choose $\p_\mu$ at distance $\ell$ from $\p_{\mathrm{ref}}$ in a uniform direction $\e\in S^{n-1}$:
\begin{align}
    \p_\mu&=\exp_{\p_{\mathrm{ref}}}(\ell\mathbf{J}_{\p_{\mathrm{ref}}}\e),
    &\ell&=\begin{cases}
        0.3\iota, & \iota<\infty, \\
        1, & \iota=\infty
    \end{cases}
\end{align}

For a point $\p$ in the geodesic ball of radius $\iota$ around $\p_{\mathrm{ref}}$, let $s$ be a normal coordinate chart at $\p_{\mathrm{ref}}$ and $x$ be another coordinate system, defined by
\begin{align}
    s(\p)&=\mathbf{J}_{\p_{\mathrm{ref}}}^{-1}\log_{\p_{\mathrm{ref}}}(\p),
    & \x=x(\p)&=\Psi\!\left(s(\p)\right)
\end{align}
$\Psi\colon\R^n\to\R^n$ is fixed and invertible.
When $\iota=\infty$, these coordinates cover all of $\M$.
We use the metric $g_x=(x^{-1})^*g$ in the observed coordinates.

To construct $\Psi$, we compose four random orthogonal maps $Q_0,\ldots,Q_3$ and three elementwise nonlinear maps $f_1,f_2,f_3$ in the order
\begin{align}
    \Psi=Q_3\circ f_3\circ Q_2\circ f_2\circ Q_1\circ f_1\circ Q_0
\end{align}
Each $f_k$ applies the following invertible map to an input coordinate $c$:
\begin{align}
    c\mapsto m+s\sinh\!\left(\operatorname{arcsinh}\!\left(\frac{c-m}{s}\right)+b\right),
    \qquad s>0
\end{align}
We draw $b\sim\operatorname{Uniform}(-0.8,0.8)$ independently for each coordinate and each $f_k$.
$m$ and $s$ are chosen to standardize each coordinate based on a batch of data at initialization.
We use the same $\Psi$ for all base points and data splits.

To sample base points, we draw tangent coordinates at the data center $\p_\mu$ from
\begin{align}
    \uu_0\sim N(\zero,\sigma^2I_n), \quad \text{and }\sigma=\frac{1}{\sqrt n}
    \begin{cases}
        0.3\iota, & \iota<\infty, \\
        1, & \iota=\infty
    \end{cases}
\end{align}

Around each base point we train on a geodesic ball of radius $r_{\max}=0.25\iota$ when $\iota<\infty$, or $r_{\max}=1$ otherwise.
For finite $\iota$, we redraw $\uu_0$ until this entire ball lies inside the domain of $x$:
\begin{align}
    d_g\!\left(\p_{\mathrm{ref}},\;\exp_{\p_\mu}(\mathbf{J}_{\p_\mu}\uu_0)\right)+r_{\max}<\iota
\end{align}

We then generate $n_{\mathrm{bp}}$ base points from independent accepted draws,
\begin{align}
    \x_0=x\!\left(\exp_{\p_\mu}(\mathbf{J}_{\p_\mu}\uu_0)\right)
\end{align}
At each $\x_0$, we train $h_{\x_0}^\theta$ on the geodesic ball
\begin{align}
    \mathcal{U}_x(\x_0)=\{\x\in\M_x:d_g(\x_0,\x)<r_{\max}\}
\end{align}
To sample a point in this ball, we independently draw $\mathbf{v}\sim N(\zero,I_n)$ and $t\sim\operatorname{Uniform}(0,1)$, then set
\begin{align}
    \uu=r_{\max}t^{1/n}\frac{\mathbf{v}}{\|\mathbf{v}\|_2}, \qquad
    \x=\exp_{\x_0}\!\left(g_x(\x_0)^{-1/2}\uu\right).
    \label{eq:synthetic-conditional-sampling}
\end{align}
This generative process defines $\rho_x(\cdot\mid\x_0)$ with positive density throughout $\mathcal{U}_x(\x_0)$.
For each base point, we draw $n_{\mathrm{train}}$, $n_{\mathrm{val}}$ and $n_{\mathrm{test}}$ training, validation and test points respectively.

\begin{table}[p]
    \centering
    \makebox[\textwidth][c]{%
    \begin{minipage}{\dimexpr\textwidth+108pt\relax}
    \centering
    \caption{Overview of synthetic manifolds.  $V,W$ are tangent vectors on each manifold. Dims lists the tested intrinsic dimensions.}
    \label{tab:synthetic-metrics}
    \begingroup
    \fontsize{8.5}{11}\selectfont
    \definecolor{geometryink}{RGB}{37,57,75}
    \definecolor{geometrynote}{RGB}{243,246,248}
    \definecolor{geometrytext}{RGB}{70,79,87}
    \newcommand{\geometryspace}[1]{%
        \begingroup\renewcommand{\arraystretch}{1}%
        \begin{tabular}[t]{@{}l@{}}#1\end{tabular}\endgroup}
    \setlength{\fboxsep}{7pt}
    \setlength{\tabcolsep}{5pt}
    \renewcommand{\arraystretch}{1.3}
    \begin{tabular*}{\textwidth}{@{\extracolsep{\fill}}lcccccc@{}}
        \toprule
        \textbf{Space / element} & \textbf{Metric} & \textbf{Exp. map} & \textbf{Log. map} & \textbf{Distance} & \textbf{Inj. radius} & \textbf{Dims} \\
        \midrule
        \multicolumn{7}{@{}l@{}}{\textcolor{geometryink}{\textbf{Sphere}}} \\[2pt]
        \geometryspace{$\p_0,\p_1\in S^n$} & $V^\top W$ & $\cos(r)\p_0+\frac{\sin r}{r}V$ & $\frac{\theta}{\sin\theta}(\p_1-\cos(\theta)\p_0)$ & $\theta$ & $\pi$ & $2,30$ \\
        \addlinespace[5pt]
        \multicolumn{7}{@{}l@{}}{\colorbox{geometrynote}{%
        \begin{minipage}{\dimexpr\textwidth-2\fboxsep\relax}
        \fontsize{8}{10.5}\selectfont\color{geometrytext}\raggedright
        $\p_0^\top V=0$, $r=\|V\|_2$, $\theta=\arccos(\p_0^\top\p_1)$.
        \end{minipage}}} \\
        \addlinespace[10pt]
        \multicolumn{7}{@{}l@{}}{\textcolor{geometryink}{\textbf{Flat torus}}} \\[2pt]
        \geometryspace{$\p_0,\p_1\in\mathbb{T}^n$} & $V^\top W$ & $[\p_0+V]$ & $\operatorname{wrap}(\p_1-\p_0)$ & $\|\operatorname{wrap}(\p_1-\p_0)\|_2$ & $\pi$ & $2,30$ \\
        \addlinespace[5pt]
        \multicolumn{7}{@{}l@{}}{\colorbox{geometrynote}{%
        \begin{minipage}{\dimexpr\textwidth-2\fboxsep\relax}
        \fontsize{8}{10.5}\selectfont\color{geometrytext}\raggedright
        $\mathbb{T}^n=\R^n/(2\pi\mathbb{Z})^n$. $[\cdot]$ denotes equivalence modulo $(2\pi\mathbb{Z})^n$; $\operatorname{wrap}$ selects the representative in $[-\pi,\pi)^n$.
        \end{minipage}}} \\
        \addlinespace[10pt]
        \multicolumn{7}{@{}l@{}}{\textcolor{geometryink}{\textbf{Hyperbolic}}} \\[2pt]
        \geometryspace{$\p_0,\p_1\in\mathbb{H}^n$} & $\langle V,W\rangle_L$ & $\cosh(r)\p_0+\frac{\sinh r}{r}V$ & $\frac{\theta}{\sinh\theta}(\p_1-\cosh(\theta)\p_0)$ & $\theta$ & $\infty$ & $2,30$ \\
        \addlinespace[5pt]
        \multicolumn{7}{@{}l@{}}{\colorbox{geometrynote}{%
        \begin{minipage}{\dimexpr\textwidth-2\fboxsep\relax}
        \fontsize{8}{10.5}\selectfont\color{geometrytext}\raggedright
        $\mathbb{H}^n=\{\p_0\in\R^{n+1}:\langle\p_0,\p_0\rangle_L=-1,\ p_0^{n+1}>0\}$.\\[2pt] $\langle V,W\rangle_L=\sum_{i=1}^n V^iW^i-V^{n+1}W^{n+1}$.\\[2pt] Tangents satisfy $\langle\p_0,V\rangle_L=0$; $r=\sqrt{\langle V,V\rangle_L}$, $\theta=\operatorname{arccosh}(-\langle\p_0,\p_1\rangle_L)$.
        \end{minipage}}} \\
        \addlinespace[10pt]
        \multicolumn{7}{@{}l@{}}{\textcolor{geometryink}{\textbf{SPD matrices}}} \\[2pt]
        \geometryspace{$\p_0,\p_1\in\mathrm{SPD}(p)$} & $\tr(\p_0^{-1}V\p_0^{-1}W)$ & $\p_0^{1/2}\exp(A)\p_0^{1/2}$ & $\p_0^{1/2}\log(B)\p_0^{1/2}$ & $\|\log(B)\|_F$ & $\infty$ & $3,28$ \\
        \addlinespace[5pt]
        \multicolumn{7}{@{}l@{}}{\colorbox{geometrynote}{%
        \begin{minipage}{\dimexpr\textwidth-2\fboxsep\relax}
        \fontsize{8}{10.5}\selectfont\color{geometrytext}\raggedright
        $\p_0=\p_0^\top\succ0$, with symmetric tangents $V,W$; dimension $p(p+1)/2$. We use $p=2,7$.\\[2pt] $A=\p_0^{-1/2}V\p_0^{-1/2}$, $B=\p_0^{-1/2}\p_1\p_0^{-1/2}$.
        \end{minipage}}} \\
        \addlinespace[10pt]
        \multicolumn{7}{@{}l@{}}{\textcolor{geometryink}{\textbf{Grassmannian}}} \\[2pt]
        \geometryspace{$[\p_0],[\p_1]\in\operatorname{Gr}(k,p)$} & $\tr(V^\top W)$ & $\exp(\Omega)\p_0$ & $A\arctan(\Sigma)B^\top$ & $\|\arctan(\Sigma)\|_F$ & $\pi/2$ & $4,30$ \\
        \addlinespace[5pt]
        \multicolumn{7}{@{}l@{}}{\colorbox{geometrynote}{%
        \begin{minipage}{\dimexpr\textwidth-2\fboxsep\relax}
        \fontsize{8}{10.5}\selectfont\color{geometrytext}\raggedright
        $\p_0\in\R^{p\times k}$ represents a subspace, $\p_0^\top \p_0=I_k$, $\p_0^\top V=0$; dimension $k(p-k)$. We use $(k,p)=(2,4),(5,11)$.\\[2pt] $\Omega=V\p_0^\top-\p_0 V^\top$; thin SVD $A\Sigma B^\top=(I_p-\p_0\p_0^\top)\p_1(\p_0^\top \p_1)^{-1}$, with $\arctan$ acting on the diagonal of $\Sigma$.
        \end{minipage}}} \\
        \addlinespace[10pt]
        \multicolumn{7}{@{}l@{}}{\textcolor{geometryink}{\textbf{Rotation group}}} \\[2pt]
        \geometryspace{$\p_0,\p_1\in\mathrm{SO}(n)$} & $\tr(V^\top W)$ & $\p_0\exp(\p_0^\top V)$ & $\p_0\log(\p_0^\top \p_1)$ & $\|\log(\p_0^\top \p_1)\|_F$ & $\pi\sqrt{2}$ & $3,28$ \\
        \addlinespace[5pt]
        \multicolumn{7}{@{}l@{}}{\colorbox{geometrynote}{%
        \begin{minipage}{\dimexpr\textwidth-2\fboxsep\relax}
        \fontsize{8}{10.5}\selectfont\color{geometrytext}\raggedright
        $\p_0^\top \p_0=I_n$, $\det \p_0=1$; dimension $n(n-1)/2$. We use $n=3,8$. Tangents have the form $V=\p_0K$, $K^\top=-K$.
        \end{minipage}}} \\
        \addlinespace[10pt]
        \bottomrule
    \end{tabular*}
    \endgroup
    \end{minipage}}
\end{table}

\subsection{RBF Metric}
\label{sec:rbf-metric}
The synthetic datasets of the exponential map experiment in \Cref{app:exp-map-experiments} use the RBF metric of metric flow matching (MFM) \citep{kapusniak2024metric}.
This metric is constructed so that geodesics avoid regions without data.
\begin{align}
    g_x(\x) = \frac{1}{h(\x) + \rho}\, I_n, \qquad
    h(\x) = \sum_{k=1}^{K} w_k \exp\left(-\frac{\lambda_k}{2}\|\x - \mathbf{c}_k\|_2^2\right)
\end{align}
where $h$ is a radial basis function network with $K = 50$ centers $\mathbf{c}_k$, bandwidths $\lambda_k$ and weights $w_k$, and $\rho > 0$ is an offset.
Although \citet{kapusniak2024metric} state a more general diagonal metric, their released code fits a single network and applies it to every coordinate.

The centers are selected as the $K$-means centroids of the training set.
Let $\sigma_k$ denote the root mean square spread of the $k$-th cluster per coordinate.
We set $\lambda_k = 1/(2(\kappa\sigma_k)^2)$ with $\kappa = 1.5$.

Let $\x^{(1)},\ldots,\x^{(M)}$ denote $M = 50{,}000$ points drawn at random from the training set.
We select the weights by solving the following optimization problem
\begin{align}
    \min_{w_1,\ldots,w_K \geq 10^{-4}} J, \qquad
    J = \frac{1}{M}\sum_{i=1}^{M}\big(1 - h(\x^{(i)})\big)^2.
\end{align}

The fit makes $h \approx 1$ on the data, so $g_x$ is close to $(1 + \rho)^{-1} I_n$ there and grows to $\rho^{-1} I_n$ away from it.
This makes paths that leave the data have large lengths, so geodesics will tend to stay close to the data.
We set $\rho = (1 - J)/2.75$ following MFM, where $J$ is evaluated at the fitted weights and averaged over the validation set in place of the $M$ training points.

\subsection{Diffusion geometry}
\label{sec:diffusion-geometries}

\Cref{fig:four-diffusion-charts}(a) was built with non-parametric diffusion geometries built from 2d point clouds \citep[Section 3.1]{jones2026computing}.
The pointclouds we use come from $m=50{,}000$ samples from three toy datasets depicting an uneven concentration on a disk, two moons, and nested circles \Cref{app:datasets-and-rmm-geometries}.
The disk has density proportional to $\exp(-1.8x^1)$ on $\|\x\|_2<1$ and both the two moons and nested circles have Gaussian noise standard deviations $0.06$ and $0.04$, respectively, with circle radii $1$ and $0.5$.

Let $\{\x^{(j)}\}_{j=1}^{m} \subset \R^n$ denote the point cloud.
We build a Markov chain over the cloud from a variable-bandwidth Gaussian kernel supported on each point's $32$ nearest neighbors,
\begin{align}
    k(\x,\y) = \exp\left(-\frac{\|\x - \y\|_2^2}{\epsilon\, \rho(\x)\rho(\y)}\right), \qquad
    \tilde{\mathbf{K}}_{jl} = \frac{k(\x^{(j)}, \x^{(l)})}{\big(q(\x^{(j)})\, q(\x^{(l)})\big)^{\nicefrac{(2 - \hat{d})}{4}}}, \qquad
    \mathbf{P}_{jl} = \frac{\tilde{\mathbf{K}}_{jl}}{\sum_{l'=1}^m \tilde{\mathbf{K}}_{jl'}},
\end{align}
where $\rho(\x) = q(\x)^{-\nicefrac{1}{2}}$ is the bandwidth function and $q$ is a kernel density estimate.
The bandwidth of $q$ at each point is the root-mean-square distance to the point's seven nearest neighbors.
We choose the scale $\epsilon$ and intrinsic dimension estimate $\hat{d}$ jointly following the procedure of \citet{jones2026computing}.
For each $\epsilon \in \{2^{-10},2^{-9.75},\ldots,2^{9.75}\}$, we average the kernel weights and compute the slope of this average on a log-log scale.
We select the value of $\epsilon$ with the largest slope and set $\hat{d}$ to twice that slope.
Since the local bandwidths depend on the sampling density, we first use this procedure to estimate the density, update the bandwidths, and then repeat it to obtain the values of $\epsilon$ and $\hat{d}$ used in $\tilde{\mathbf{K}}$.
The symmetrized row sums of $\mathbf{P}$, normalized to sum to one, form the measure $\boldsymbol{\mu}$ over the cloud.

We compute the leading, $n_0=240$ for the disk and $n_0=120$ for the other two datasets, eigenfunctions of $\mathbf{P}$ and store their values on the cloud in the $m \times n_0$ matrix $\mathbf{U}$, which satisfies $\mathbf{U}^T \diag(\boldsymbol{\mu}) \mathbf{U} = \mathbf{I}_{n_0}$.
The eigenfunctions are the smoothest functions on the cloud, so replacing the values $\mathbf{x}_c \in \R^m$ of each coordinate function by $\mathbf{U}\mathbf{U}^T \diag(\boldsymbol{\mu})\, \mathbf{x}_c$ removes its noisiest components \citep[Computational Note 7]{jones2026computing}.

The geometry corresponding to the point cloud is given through the \emph{carr\'e du champ} of the smoothed coordinates.
It is a covariance under one step of the Markov chain \citep[Equation 6]{jones2026computing} given by,
\begin{align}
    \Gamma_j(x_c, x_{c'}) = \frac{2}{\epsilon\, \rho(\x^{(j)})^2} \mathrm{Cov}\big[x_c(X),\, x_{c'}(X) \,:\, X \sim \mathbf{P}_{j}\big].
\end{align}
We average the per-point tensors once under $\mathbf{P}$, symmetrize them, and divide all of them by the median over the cloud of the largest eigenvalue.

Finally, the cometric is defined as a soft, Gaussian weighted average of the CDC over the point cloud
\begin{align}
    \label{eq:diffusion-cometric}
    g_x(\x)^{-1} &= \delta I_n + (1-\delta)a(\x)\sum_{j=1}^m w^{(j)}(\x) \Gamma_j(\x, \x), \notag \\
    w^{(j)}(\x) &= \frac{\exp\big(-\|\x - \x^{(j)}\|_2^2 / 2\sigma^2\big)}{\sum_{l=1}^m \exp\big(-\|\x - \x^{(l)}\|_2^2 / 2\sigma^2\big)}
\end{align}
where $\delta = 2\times 10^{-3}$ and $\sigma$ is $0.08$ for the disk and the median distance to the eigth neighbor for the two moons and nested circles datasets.
$a(\x)$ is an extra factor to have the geometry behave well away from the support of the data.
For two moons and nested circles, $a(\x)=1$ and for the disk,
\begin{align}
    a(\x)=\frac{S(\x)}{S(\x)+\tau},\qquad
    S(\x)=\sum_{j=1}^m\exp\left(-\frac{\|\x-\x^{(j)}\|_2^2}{2\sigma_s^2}\right),
\end{align}
with $(\sigma_s,\tau)=(0.02,1)$.

\paragraph{Chart fitting for the diffusion illustration}
For \Cref{fig:four-diffusion-charts}, we fit one chart per displayed base point to its $4{,}096$ nearest neighbors within the same connected component, excluding the base point.
Each chart has four affine coupling layers with dense linear mixing, two gated conditioner blocks of width $32$, and a one-block base point encoder of width $16$ with $8$ output features.
We use metric preconditioning and center inputs at the base point, scaling by the median neighbor distance measured with $g_x(\x_0)$.
We minimize the normal loss with Adam at a constant learning rate of $3\times10^{-4}$ and unit gradient clipping, using FP64 chart parameters and batches of $64$ endpoints ($128$ for the disk).
Let $r_0$ be the median encoded neighbor norm at initialization.
We train for $5{,}000$ steps with $\gamma=0.1r_0/\sqrt{2}$, then restart Adam from the best checkpoint with $\gamma=0.5r_0$ for another $2{,}000$ steps on moons and circles or $5{,}000$ on the disk.
Every $100$ steps we evaluate fixed independent bridge draws on the same training endpoints and retain the lowest-loss checkpoint for the current $\gamma$.
The cloud seed is $0$ and the three chart initialization seeds are $0,1,2$.

\paragraph{Normalizing-flow pullback illustration}
For the pullback panel of \Cref{fig:four-diffusion-charts}, we fit an unconditional flow $h$ to the uneven-concentration dataset by maximum likelihood under a standard Gaussian prior, using the four-layer affine coupling architecture described above.
We compose its encoder with the componentwise standard Gaussian CDF, $H=\Phi\circ h$, and define the pullback metric $g_x=DH^\top DH$, whose volume density equals the fitted data density.
We freeze this metric and fit one flow chart at each of the same three base points, using the diffusion illustration's chart architecture and training procedure with FP32.

For both geometries, we decode coordinate grids restricted to disks of a common radius within each panel, orient the grids using rotation-only ICA on encoded neighbors, and overlay the same $600$ observations.
The backgrounds use separate logarithmic color scales: blue shows inverse volume density for the diffusion metric, and green shows volume density for the pullback metric.
Panels (b) and (d) show normalized histograms of the learned coordinate norms of each base point's $4{,}096$ nearest neighbors, using $40$ equally spaced bins shared across base points within each geometry.
Colors identify corresponding base points across the geometry panels and histograms.

\subsection{Riemannian metric matching}
\label{sec:rmm-geometries}
For higher dimensional datasets, we use Riemannian metric matching (RMM) \citep{bamberger2026metricmatching} to learn its geometry.
Let $p_{\varepsilon|0}(\x_\varepsilon \mid \x) = \mathcal{N}(\x_\varepsilon \mid \x, \varepsilon I_n)$ denote the transition distribution at time $\varepsilon$ of Brownian motion started at $\x$, let $p_{0,\varepsilon}(\x, \x_\varepsilon) = p_{\text{data}}(\x)\, p_{\varepsilon|0}(\x_\varepsilon \mid \x)$ denote the joint distribution of a data point and its noised version, let $p_\varepsilon(\x_\varepsilon)$ denote the marginal of the noised point, and let $p_{0|\varepsilon}(\cdot \mid \x_\varepsilon)$ denote the posterior over the clean point given the noised one.
The posterior mean is the denoiser
\begin{align}
    \bar{\x}_0(\x_\varepsilon) = \mathbb{E}_{\x_0 \sim p_{0|\varepsilon}(\cdot \mid \x_\varepsilon)}\left[\x_0\right]
\end{align}
and the geometry is the family of cometrics, indexed by $\varepsilon$, given by the scaled posterior covariance
\begin{align}
    \label{eq:posterior-cometric}
    g_{\varepsilon,x}(\x_\varepsilon)^{-1} = \frac{1}{2\varepsilon}\, \mathrm{Cov}_{\x_0 \sim p_{0|\varepsilon}(\cdot \mid \x_\varepsilon)}\left[\x_0\right] = \frac{1}{2}\, D\bar{\x}_0(\x_\varepsilon)
\end{align}
where the second equality is Tweedie's formula for the posterior covariance \citep{efron2011tweedie}.
This is the mean centered carr\'e du champ of RMM.
At a noised point $\x_\varepsilon$ the covariance spans the directions along which clean data that could have produced $\x_\varepsilon$ vary, so the cometric is aligned with the data manifold.
RMM fits a parametric cometric, $g_{\varepsilon,x}^\theta$, to the target by minimizing the following regression:
\begin{align}
    \label{eq:metric-matching}
    \mathcal{L}(\theta)
    &= \underset{\varepsilon \sim p(\varepsilon)}{\mathbb{E}} \, \underset{\x_\varepsilon \sim p_\varepsilon}{\mathbb{E}}\left[\left\|g_{\varepsilon,x}^\theta(\x_\varepsilon)^{-1} - g_{\varepsilon,x}(\x_\varepsilon)^{-1}\right\|_F^2\right] \\
    &= \underset{\varepsilon \sim p(\varepsilon)}{\mathbb{E}} \, \underset{(\x, \x_\varepsilon) \sim p_{0,\varepsilon}}{\mathbb{E}}\left[\left\|g_{\varepsilon,x}^\theta(\x_\varepsilon)^{-1} - \frac{1}{2\varepsilon}\left(\x - \bar{\x}_0(\x_\varepsilon)\right)\left(\x - \bar{\x}_0(\x_\varepsilon)\right)^\top\right\|_F^2\right] + \text{const}. \notag
\end{align}
The second expression is tractable to optimize provided that we can compute the denoiser.
For this, we either pretrain a denoiser, or download a pretrained denoiser from the internet, which we denote by $D_\varepsilon(\x_\varepsilon) \approx \bar{\x}_0(\x_\varepsilon)$.

\subsubsection{Low rank parameterization}
\label{app:low-rank-rmm-parameterization}
We parameterize the cometric using the approximately low rank parameterization of RMM:
\begin{align}
    g_{\varepsilon,x}^\theta(\x_{\varepsilon})^{-1} = \delta I_n + (1 - \delta)M_{\varepsilon}^\theta(\x_{\varepsilon})^\top M_{\varepsilon}^\theta(\x_{\varepsilon})
\end{align}
where $M_{\varepsilon}^\theta(\x_{\varepsilon}) \in \mathbb{R}^{d \times n}$ is the unconstrained output of a neural network, and $d \ll n$.

We train with $\delta$ incorporated into the objective, which yields
\begin{align}
    \mathcal{L}(\theta)
    &= \underset{\varepsilon \sim p(\varepsilon)}{\mathbb{E}} \, \underset{(\x, \x_\varepsilon) \sim p_{0,\varepsilon}}{\mathbb{E}}\Bigl[
    (1-\delta)^2 \|M_{\varepsilon}^\theta(\x_\varepsilon)M_{\varepsilon}^\theta(\x_\varepsilon)^\top\|_F^2
    + 2\delta(1-\delta)\|M_{\varepsilon}^\theta(\x_\varepsilon)\|_F^2 \notag\\
    &\qquad\qquad\qquad\qquad\qquad\qquad\qquad - \frac{1-\delta}{\varepsilon}\bigl\|M_{\varepsilon}^\theta(\x_\varepsilon)\bigl(\x-D_\varepsilon(\x_\varepsilon)\bigr)\bigr\|_2^2
    \Bigr] + \text{const.}
\end{align}

\subsubsection{Denoiser architecture}
\label{app:rmm-denosier-architecture}
We train a denoiser for alanine dipeptide.
The denoiser is parameterized by a residual network, and predicts the clean point under additive Gaussian corruption, so its noise argument is the variance $\varepsilon$ used by RMM.
The denoiser uses the parameterization of \citet{karras2022elucidating}, where $D_\phi(\x_\varepsilon, \varepsilon) = c_{\text{skip}}(\varepsilon)\x_\varepsilon + c_{\text{out}}(\varepsilon) F_\phi(c_{\text{in}}(\varepsilon)\x_\varepsilon, \varepsilon)$.
Here, $F_\phi$ is a trainable neural network conditioned on the noise variance, and the fixed functions $c_{\text{in}}$, $c_{\text{out}}$, and $c_{\text{skip}}$ scale its input, its output, and the skip connection, respectively.
We set the reference standard deviation of the coordinates to $\sigma_{\text{data}}=1$, since every dataset is standardized to unit root mean square spread, giving
\begin{align}
    c_{\text{in}}(\varepsilon) = \frac{1}{\sqrt{\varepsilon+1}}, \qquad
    c_{\text{out}}(\varepsilon) = \sqrt{\frac{\varepsilon}{\varepsilon+1}}, \qquad
    c_{\text{skip}}(\varepsilon) = \frac{1}{\varepsilon+1}
\end{align}

The network $F_\phi$ is a linear layer from $\R^n$ to width $w$, followed by $L$ residual blocks of width $w$, and a linear layer back to $\R^n$ whose weights are initialized to zero so that the denoiser starts as the identity map.
The embedding for $\varepsilon$, $\mathbf{e}(\varepsilon) \in \R^{c}$, is built from $32$ Fourier features of $\log\varepsilon$ at the frequencies $1000^{-k/32}$ for $k = 0, \dots, 31$, whose cosines and sines are passed through two linear layers with a SiLU between them.
Each residual block is conditioned on the embeddings using FiLM \citep{perez2018film},
\begin{align}
    \mathbf{h} \mapsto \mathbf{h} + W_2\,\mathrm{SiLU}\!\big((\one + \mathbf{a})\odot\mathrm{SiLU}(W_1\mathbf{h}) + \mathbf{b}\big), \qquad (\mathbf{a}, \mathbf{b}) = W_c\,\mathrm{SiLU}(\mathbf{e}(\varepsilon))
\end{align}
where $W_2$ is initialized to zero so that each block also starts as the identity.
We use $w = 512$, $L = 8$, and $c = 128$, which is about $5$M parameters.
See \Cref{tab:vector-model-architectures} for the denoiser and cometric architectures and training details.

\paragraph{Denoiser training}
We train the denoiser on the training split with the denoising loss of \citet{karras2022elucidating}, which weights the squared error by $c_{\text{out}}(\varepsilon)^{-2}$,
\begin{align}
    \mathcal{L}(\phi) = \underset{\varepsilon}{\mathbb{E}}\,\underset{(\x, \x_\varepsilon) \sim p_{0,\varepsilon}}{\mathbb{E}}\left[c_{\text{out}}(\varepsilon)^{-2}\left\|D_\phi(\x_\varepsilon, \varepsilon) - \x\right\|_2^2\right]
\end{align}
We draw $\varepsilon$ log-uniformly over $[0.01, 5]$.
We train with Adam, clip gradients to unit global norm, and ramp the learning rate linearly over the first $1{,}000$ steps.
We train the denoiser for a maximum of $1{,}000{,}000$ steps.
We keep an exponential moving average of the weights with decay $0.999$, evaluate its loss on the validation split every $1{,}000$ steps with early stopping enabled, and select the averaged checkpoint with the lowest validation loss.
A run stops once $50{,}000$ steps pass without an improvement in the validation loss.
We sweep the peak learning rate over $\{5\times10^{-5}, 10^{-4}, 2\times10^{-4}, 4\times10^{-4}, 8\times10^{-4}\}$ with three seeds each, and select the run whose selected checkpoint has the smallest loss on the validation split.

\begin{table}[t]
    \centering
    \caption{Denoiser and cometric architectures and training on alanine dipeptide.}
    \label{tab:vector-model-architectures}
    {\small
    \setlength{\tabcolsep}{8pt}
    \begin{tabular}{lc}
        \toprule
        Property & Alanine dipeptide \\
        \midrule
        \multicolumn{2}{l}{\textbf{Residual network architecture}} \\
        Dimension $n$ & $30$ \\
        Width $w$ & $512$ \\
        Residual blocks $L$ & $8$ \\
        Embedding size $c$ & $128$ \\
        Total parameters (M) & $5.32$ \\
        \midrule
        \multicolumn{2}{l}{\textbf{Denoiser}} \\
        Training points & $525{,}000$ \\
        Validation points & $75{,}000$ \\
        Test points & $150{,}000$ \\
        Noise range & $[0.01, 5]$ \\
        Batch size & $512$ \\
        Maximum training steps & $1{,}000{,}000$ \\
        EMA decay & $0.999$ \\
        Selected learning rate & $5\times10^{-5}$ \\
        Selected step & $245{,}000$ \\
        Validation loss & $6.336$ \\
        Test loss & $6.361$ \\
        \midrule
        \multicolumn{2}{l}{\textbf{Cometric}} \\
        Rank $d$ & $30$ \\
        Floor $\delta$ & $10^{-4}$ \\
        Total parameters (M) & $5.76$ \\
        Selected learning rate & $5\times10^{-5}$ \\
        Selected step & $251{,}000$ \\
        Validation loss & $-0.7793$ \\
        Test loss & $-0.7762$ \\
        \bottomrule
    \end{tabular}
    }
\end{table}

\paragraph{Cometric architecture}
The cometric shares the exact same architecture as the denoiser, except that the last layer maps to a larger vector of size $\R^{d\times n}$ to produce the matrix $M_\varepsilon^\theta$, following \Cref{app:low-rank-rmm-parameterization}.
The cometric does not use the coefficients $c_{\text{in}}$, $c_{\text{out}}$, and $c_{\text{skip}}$, so $M_\varepsilon^\theta(\x_\varepsilon)$ is the output of the network evaluated at $\x_\varepsilon$ directly.
We use a full rank parameterization with $d=n$ and set $\delta = 1\times 10^{-4}$.

\paragraph{Cometric training}
We train the cometric on the training data with the objective from \Cref{app:low-rank-rmm-parameterization}, centered on the selected denoiser.
The optimization procedure is identical to how we trained the denoiser.

\section{Experiment details}
\raggedbottom
\label{app:experiemnt-details}
\label{app:protocol}

Here we give the details of our experiments.

\paragraph{Timing details}
\label{app:timing-details}
We measure wall clock time for all of the experiments in this paper in the same way.
This method takes a compiled JAX program as input, with its inputs already in device memory, so that compilation time and other preprocessing steps are not included in the reported time.
The program is run once as a warmup, and then $5$ times after, from which we report the median time.
All timings are reported on one NVIDIA A16 in float32 and every program is implemented in JAX \citep{jax2018github} and compiled by XLA with its default settings.

\subsection{Logarithmic map experiment details}
\label{sec:normal-recovery-protocol}

We generate data on each manifold following the procedure in \Cref{sec:synthetic-metrics}.
For each manifold, we generate $2048$ basepoints for training the amortized models, $64$ for validation and $128$ for testing.
For each base point, we sample $1024$ associated endpoints to train the MFM and flow chart with.
These methods are then evaluated on endpoints of the $128$ test base points.
All methods use the same test endpoints.
The closed form manifolds provide exact references.
The data geometries use the reference solutions described in \Cref{app:exp-map-experiments}.
After reference filtering, the two moons has queries at $124$ test base points and the other geometries have queries at all $128$.
We report the mean of $\Rlog(\x_0)$ and its standard error across the test base points, separately for each manifold and dimension.

\subsubsection{Methods}
\label{sec:geodesic-interpolation-based-solvers}
Our baselines employ geodesic interpolation to retrieve normal coordinate and logmap estimates.
These methods parameterize a curve $\gamma_{(\cdot)}^\phi\colon [0,1] \to \M_x$ that starts at a base point $\x_0$ and ends at a query point $\x$, and approximate a geodesic by minimizing its energy or solving the geodesic equation.
The iterative algorithms (GP and GEORCE) are run until the relative decrease of their energy between iterations falls below a tolerance $\epsilon_{\text{stop}} = 1\times 10^{-5}$, and are all parameterized and optimized in float32.
Each solver is stopped at \mbox{$1{,}000$} max iterations, and fails if its solution is nonfinite or the solver reaches this cap without meeting its stopping criterion.
Failed queries are left out of both the wall clock time and the errors.

\paragraph{GP}
We implement the fixed-point solver of \citet{arvanitidis2019fast}, which represents a curve by a Gaussian-process posterior mean conditioned on its endpoints and $K$ acceleration vectors at uniformly spaced times between $0$ and $1$, with the endpoints included.
We use a straight-line prior mean, a squared-exponential kernel with length scale $\lambda = \sqrt{1/(2(K-1))}$, and for numerical stability in FP32, acceleration regularization of $10^{-2}$.
The model iteratively infers the acceleration of the posterior mean of a Gaussian process, that is conditioned on the endpoint conditions and the accelerations.

\paragraph{GEORCE}
We also use the GEORCE \citep{rygaard2026georce} algorithm to solve for a geodesic interpolation.
This approach parameterizes the geodesic with $K$ knot points, but uses an iterative algorithm to solve for the knot points without assuming any particular curve parameterization.
GEORCE recovers the geodesic interpolation by solving the discretized energy as a control problem with boundary constraints.
GEORCE returns $K$ knot points rather than a curve, so we interpolate its knots with the piecewise cubic Hermite interpolant with backward differences of \citet{morrill2021cdeonline} for evaluation.
The initial velocity of GEORCE is the forward difference $(K+1)(\y_1-\x_0)$, where $\y_0=\x_0,\y_1,\dots,\y_K,\y_{K+1}=\x$ at the times $t_k=k/(K+1)$ are the knots that are solved for with the algorithm.

\paragraph{Metric flow matching (MFM)}
Our neural network baseline is MFM \citep{kapusniak2024metric}.
This method parameterizes each curve as follows:
\begin{align}
    \x_{\text{curve},t}^{\theta} = (1-t)\x_0 + t\x_1 + t(1-t)\mathrm{NN}_\theta(t,\x_0,\x_1)
\end{align}
where $\mathrm{NN}_\theta$ is a neural network.
We construct $\mathrm{NN}_\theta$ as a gated residual network.
MFM is trained to minimize the energy $E$.
The training procedure is identical to the flow chart's, below.

We choose the number of blocks and hidden size separately for two MFM baselines.
The first has a parameter count modestly close to that of the four-layer flow chart, and the second has a comparable time to compute the fast normal coordinate estimates $\z^{(i)}$ of \Cref{app:synthetic-experiment-eval-metrics} for a batch of $1024$ queries.
For the flow chart, this is $h_{\x_0}^\theta(\x^{(i)})$, and for MFM it is $\mathbf{J}_x^{-1}\dot{\x}_0^{(i)}$.
\Cref{tab:amortized-model-sizes} lists the model sizes and computation times.
Times are microseconds per query on an NVIDIA A16 in FP32.

\begin{table}[t]
    \centering
    \caption{Sizes of the amortized models and the wall clock cost of computing 1024 normal coordinate estimates at a fixed base point.}
    \label{tab:amortized-model-sizes}
    {
    \small
    \setlength{\tabcolsep}{5pt}
    \begin{tabular}{r r r r r r r r r}
        \toprule
        & \multicolumn{2}{c}{Flow chart} & \multicolumn{3}{c}{MFM, parameter matched} & \multicolumn{3}{c}{MFM, time matched} \\
        \cmidrule(lr){2-3} \cmidrule(lr){4-6} \cmidrule(lr){7-9}
        $n$ & Parameters & $\mu$s & Width $\times$ blocks & Parameters & $\mu$s & Width $\times$ blocks & Parameters & $\mu$s \\
        \midrule
        $2$  & $187{,}504$ & $0.98$ & $64 \times 15$ & $188{,}674$ & $1.51$ & $64 \times 9$ & $113{,}410$ & $1.01$ \\
        $3$  & $191{,}296$ & $1.00$ & $64 \times 15$ & $188{,}867$ & $1.71$ & $64 \times 6$ & $75{,}971$ & $0.97$ \\
        $4$  & $195{,}104$ & $1.05$ & $96 \times 7$ & $197{,}572$ & $1.39$ & $64 \times 6$ & $76{,}164$ & $1.01$ \\
        $28$ & $291{,}296$ & $1.33$ & $96 \times 10$ & $288{,}604$ & $2.05$ & $64 \times 6$ & $80{,}796$ & $1.32$ \\
        $30$ & $299{,}728$ & $1.35$ & $128 \times 6$ & $309{,}790$ & $1.58$ & $64 \times 9$ & $118{,}814$ & $1.30$ \\
        \bottomrule
    \end{tabular}
    }
\end{table}

\paragraph{Flow chart}
\label{sec:amortized-methods}
\label{sec:amortized-training}
\label{app:models}
The flow chart used for the amortized MFM comparison and the preconditioner is a conditional neural spline flow \citep{durkan2019neural} with $4$ rational quadratic spline coupling layers, each with $8$ bins on $[-10, 10]$ in both low and high dimensions, that is initialized to be close to the identity transformation.
Before each coupling layer, we place a parametric elementwise scale and shift that is conditioned on the base point encoder and a dense invertible linear map (with an LU parameterization).
Each coupling layer parameterizes its spline using a FiLM conditioned \citep{perez2018film} residual network with three residual blocks of width $32$.
The conditioning is computed with a separate residual network that encodes the base point into a $64$-dimensional feature vector for all coupling layers with $8$ gated residual blocks of width $64$.
The fixed-base-point and runtime experiments use the wider coupling networks described in \Cref{app:capacity}.
We train our method with the normal loss in \Cref{proposition:normal-loss} under the sampling distribution of \Cref{sec:normal-flow-chart}, with $\gamma = 0.1\, r_{\text{max}}/\sqrt{n}$.

We minimize this loss with Adam in single precision at a learning rate of $3\times 10^{-4}$ with linear warmup over the first $100$ steps.
Each training batch holds $16$ base points with $32$ endpoints each, for a total of $512$ pairs, so the base point encoder is evaluated $16$ times per step and the coupling layers on every pair.
For each pair we draw one time $t \sim \mathrm{U}(0,1)$ and $\z_t$ according to the sampling distribution in \Cref{sec:normal-flow-chart}.
Gradient updates are clipped to have unit norm.
We keep an exponential moving average of the parameters with decay $0.999$, and this average is the model we evaluate and report.
The validation loss is evaluated on held out pairs every $500$ steps.
The learning rate is constant until step $1{,}000{,}000$ and then decays linearly to zero at step $2{,}000{,}000$, where training ends, and we keep the average with the lowest validation loss.

\subsubsection{Evaluation metrics}
\label{app:synthetic-experiment-eval-metrics}
Recall that each method we test takes as input a base point $\x_0$, and query points, $\{\x^{(1)},\dots,\x^{(N)}\}$, and produces a set of coordinates, $\{\z^{(1)},\dots,\z^{(N)}\}$, and tangent vectors, $\{U_x^{(1)}, \dots, U_x^{(N)}\}$, representing the normal coordinates and log map solutions.
In this section, we review how each method produces the coordinates and tangent vector solutions.
The baseline methods, and our method, offer a curve $\x_{\cdot}^{(i)}$ from the base point to a query point $\x^{(i)}$ at $t=1$, that after training approximates a geodesic.
We use each of these curves to construct a fast logmap and normal coordinate solution.

\paragraph{Method solutions}
The fast logmap estimate is simply the initial velocity of the curve, and the corresponding normal coordinate is the logmap scaled by a $g$-orthogonal matrix.
Let $\mathbf{J}_x = g_x(\x_0)^{-\frac{1}{2}}$ be this matrix.
Then the fast logmap and normal coordinate solutions corresponding to a method applied to base point $\x_0$ and query $\x^{(i)}$ are
\begin{align}
    U_x^{(i)} = \dot{\x}_0^{(i)}, \qquad \z^{(i)} = \mathbf{J}_x^{-1}\dot{\x}_0^{(i)}
\end{align}
For the flow chart, whose curve is $\x_t^{(i)} = f_{\x_0}^\theta(t\,h_{\x_0}^\theta(\x^{(i)}))$, the fast normal coordinate is the chart's own output, $\z^{(i)} = h_{\x_0}^\theta(\x^{(i)})$, and the fast logmap is $U_x^{(i)} = Df_{\x_0}^\theta(\zero)\z^{(i)}$ as described in \Cref{sec:background-riemann-normal-coordinates}.

The GP and GEORCE methods fail if the solution is not finite, or the solver stops before meeting its stopping criterion within the maximum number of iterations.
We write \mbox{$\mathcal{I}(\x_0)\subseteq\{1,\dots,N\}$} to denote the indices of the queries of base point \mbox{$\x_0$} at which the method does not fail, and $\mathcal{J}(\x_0)$ for the indices of queries with valid reference solutions.
MFM and the flow chart fail at a query only if its solution is not finite.
Each method is graded on its own successful queries.
The failed fraction is below $0.72\%$ for each numerical logmap setting, and no evaluated chart logmap query fails.

\paragraph{Evaluation metrics}
We evaluate the solutions against the ground truth coordinates and logarithmic maps, $\{\s^{(1)}, \dots, \s^{(N)}\}$ and $\{V_x^{(1)}, \dots, V_x^{(N)}\}$.

\paragraph{Wallclock}
Since we are concerned with speed, we time how long it takes (in wall clock time) to compute the solutions, as measured by \Cref{app:timing-details}.
We let $\tau(\x_0)$ denote the wall clock time to compute the solutions $\z^{(i)}$ and $U_x^{(i)}$ for every $i\in\mathcal{I}(\x_0)$, divided by $|\mathcal{I}(\x_0)|$, so that $\tau(\x_0)$ is a time per query.
Since $\mathcal{I}(\x_0)$ is not known in advance, we first run the method on all $N$ queries of $\x_0$ to find it, and then time a second run on the queries in $\mathcal{I}(\x_0)$ alone.

For the GEORCE and GP baselines, $\tau(\x_0)$ includes the time to solve for the curve.
$\tau(\x_0)$ does not include the time to pretrain the flow chart and MFM on the training set.
For the flow chart, the time to compute the base point encoding is included in the time.

Because a batch of queries iterates until its slowest curve stops, we sweep over batch sizes in $\{64, 128, 256, 512, 1024\}$ for each of the GP and GEORCE, at each knot count and geometry on the first $16$ test base points and take the batch size with the smallest mean of $\tau(\x_0)$ over these base points.
The flow chart and both MFMs are evaluated at a batch size of $1024$.
On this best batch size, we record the time and peak memory usage of every test base point.
\Cref{tab:solver-batch-sizes} lists the batch sizes selected for the GP and GEORCE.

\begin{table}[t]
    \centering
    \caption{Batch sizes selected for the GP and GEORCE at each knot count $K$, the size with the smallest mean of $\tau(\x_0)$ over the first $16$ test base points.}
    \label{tab:solver-batch-sizes}
    {
    \small
    \setlength{\tabcolsep}{4pt}
    \begin{tabular}{l r r r r r r r r r}
        \toprule
        & & \multicolumn{4}{c}{GP} & \multicolumn{4}{c}{GEORCE} \\
        \cmidrule(lr){3-6} \cmidrule(lr){7-10}
        Manifold & $n$ & $K=4$ & $8$ & $16$ & $32$ & $K=4$ & $8$ & $16$ & $32$ \\
        \midrule
        Sphere & $2$ & $1024$ & $1024$ & $1024$ & $1024$ & $1024$ & $1024$ & $1024$ & $512$ \\
        Flat torus & $2$ & $1024$ & $1024$ & $1024$ & $1024$ & $1024$ & $1024$ & $1024$ & $1024$ \\
        Hyperbolic & $2$ & $1024$ & $1024$ & $1024$ & $512$ & $1024$ & $1024$ & $1024$ & $1024$ \\
        SPD & $3$ & $1024$ & $1024$ & $1024$ & $512$ & $1024$ & $512$ & $256$ & $128$ \\
        Grassmann & $4$ & $512$ & $128$ & $64$ & $64$ & $64$ & $64$ & $64$ & $64$ \\
        Rotations & $3$ & $1024$ & $1024$ & $1024$ & $64$ & $256$ & $64$ & $64$ & $64$ \\
        \midrule
        Sphere & $30$ & $64$ & $64$ & $64$ & $64$ & $64$ & $128$ & $64$ & $64$ \\
        Flat torus & $30$ & $64$ & $128$ & $64$ & $64$ & $64$ & $128$ & $128$ & $64$ \\
        Hyperbolic & $30$ & $64$ & $64$ & $64$ & $64$ & $64$ & $128$ & $64$ & $64$ \\
        SPD & $28$ & $128$ & $64$ & $64$ & $64$ & $64$ & $64$ & $64$ & $64$ \\
        Grassmann & $30$ & $64$ & $64$ & $64$ & $64$ & $64$ & $64$ & $64$ & $64$ \\
        Rotations & $28$ & $128$ & $64$ & $64$ & $64$ & $64$ & $64$ & $64$ & $64$ \\
        \bottomrule
    \end{tabular}
    }
\end{table}
We report the mean of $\tau(\x_0)$ over the test base points.

\paragraph{Logmap}
To evaluate the logmap solutions, we compute the normalized RMSE $\Rlog(\x_0)$ defined in \Cref{sec:logmap}.
We compute the root mean squared $g$-norm error between the logmap solution and the ground truth
\begin{align}
    \mathcal{E}_{\text{logmap}}(\x_0) =
    \left(\frac{1}{|\mathcal{I}(\x_0)|}\sum_{i\in\mathcal{I}(\x_0)}
    \left\|U_x^{(i)}-V_x^{(i)}\right\|_{g}^2\right)^\frac{1}{2}
\end{align}

We divide $\Elog(\x_0)$ by the reference scale $r_{\mathrm{ref}}(\x_0)$, which uses all queries in $\mathcal J(\x_0)$ and is shared across methods.
We average $\Rlog(\x_0)$ over base points with valid readings and report its mean and standard error multiplied by $100$.

\subsubsection{Accuracy comparisons}
\label{app:synthetic-experiment-reports}

\paragraph{Amortized accuracy}
In this experiment, we pretrain a flow chart and the two MFMs described in \Cref{tab:amortized-model-sizes} on the training set of each geometry, and then evaluate them on the test set.
Each is trained once, from one seed, so the error bars measure variation across base points.
For the baselines, we evaluate them directly on the test set.
We sweep the knot count of the GP and GEORCE, over $K \in \{4, 8, 16, 32\}$.

To estimate the uncertainty in the ratios of \Cref{fig:accuracy-amortized}, we resample the held out base points with replacement $10{,}000$ times, using the same base points for the flow chart and MFM in each sample.
We compute the ratio of mean errors for each sample and report the $2.5$th and $97.5$th percentiles.

\paragraph{Fixed base point accuracy}
We evaluate the numerical baselines in \Cref{fig:accuracy} at the same four base points as the non-amortized flow charts, with indices $26$, $44$, $67$ and $100$.
For each baseline, we select the knot count with the lowest mean of $\Rlog$ over all test base points, which gives $K=16$ for the GP and $K=32$ for GEORCE on all sixteen geometries.
Each flow chart is trained once, from one seed, so the error bars measure variation across base points.
The training queries at these base points are separate from the test queries, as described in \Cref{app:capacity}.

\paragraph{Convergence of the baselines}
\Cref{tab:baseline-convergence} shows the number of test queries that fail, for the GP and GEORCE at each knot count over all test queries of all test base points.

\begin{table}[t]
    \centering
    \caption{Number of test queries that fail for the GP and GEORCE at each knot count $K$, out of $131{,}072$ queries per geometry, $1{,}024$ at each of the $128$ test base points.}
    \label{tab:baseline-convergence}
    {
    \small
    \setlength{\tabcolsep}{4pt}
    \begin{tabular}{l r r r r r r r r r}
        \toprule
        & & \multicolumn{4}{c}{GP} & \multicolumn{4}{c}{GEORCE} \\
        \cmidrule(lr){3-6} \cmidrule(lr){7-10}
        Manifold & $n$ & $K=4$ & $8$ & $16$ & $32$ & $K=4$ & $8$ & $16$ & $32$ \\
        \midrule
        Sphere & $2$ & $246$ & $32$ & $3$ & $0$ & $0$ & $0$ & $0$ & $0$ \\
        Flat torus & $2$ & $0$ & $0$ & $0$ & $0$ & $0$ & $0$ & $0$ & $0$ \\
        Hyperbolic & $2$ & $0$ & $0$ & $0$ & $0$ & $0$ & $0$ & $0$ & $0$ \\
        SPD & $3$ & $0$ & $0$ & $0$ & $0$ & $0$ & $0$ & $0$ & $0$ \\
        Grassmann & $4$ & $1$ & $0$ & $0$ & $0$ & $0$ & $0$ & $0$ & $0$ \\
        Rotations & $3$ & $0$ & $0$ & $0$ & $0$ & $0$ & $0$ & $0$ & $0$ \\
        \midrule
        Sphere & $30$ & $592$ & $0$ & $0$ & $0$ & $0$ & $0$ & $0$ & $0$ \\
        Flat torus & $30$ & $18$ & $0$ & $0$ & $0$ & $0$ & $0$ & $0$ & $0$ \\
        Hyperbolic & $30$ & $151$ & $0$ & $0$ & $0$ & $0$ & $0$ & $0$ & $0$ \\
        SPD & $28$ & $199$ & $0$ & $0$ & $0$ & $0$ & $0$ & $0$ & $0$ \\
        Grassmann & $30$ & $41$ & $0$ & $0$ & $0$ & $0$ & $0$ & $0$ & $0$ \\
        Rotations & $28$ & $21$ & $0$ & $0$ & $0$ & $0$ & $0$ & $0$ & $0$ \\
        \bottomrule
    \end{tabular}
    }
\end{table}

\subsection{Exponential map experiment details}
\label{app:exp-map-experiments}

\subsubsection{Data and geometry}
\label{app:exp-map-experiments-data-and-geometry}
We use the twelve closed form manifolds of \Cref{sec:synthetic-metrics}, three geometries constructed from synthetic datasets using the MFM RBF metric, and the diffusion geometry of alanine dipeptide learned with RMM as described in \Cref{app:datasets-and-rmm-geometries}.
After pretraining each RMM geometry, we fix $\varepsilon = 0.1$, so the flow charts, the ODE solvers, and the ground truth all use the cometric $g_{0.1,x}^\theta$.

\paragraph{Synthetic datasets}
The synthetic datasets have $500{,}000$ points total, and are split into a training, validation and test set with a $(70, 10, 20)$ split.
We use the RBF metric from MFM, which we detail in \Cref{sec:rbf-metric}.
The ground truth is generated in FP64 but the experiments are run in FP32.

\paragraph{Data selection}
\label{app:neighbor-data-selection}
To select endpoints, we sample from each base point based on proximity.
For this, we first run clustering on each dataset (except for alanine dipeptide due to its size) and give each point a label corresponding to its connected component.
This clustering is done with the HDBSCAN \citep{campello2013density,mcinnes2017hdbscan} algorithm from Scikit-learn \citep{pedregosa2011scikit}.
We run it separately on each split, with a minimum cluster size of $1\%$ of the split and $16$ samples per core point, and allow a single cluster.
Points that are labeled as noise are given the label of their nearest neighbor with a label.
We give all points in alanine dipeptide the same cluster label.

Then, given a base point $\x_0$ that is drawn uniformly from the data split, we select an endpoint $\x$ by drawing uniformly from the $k=\lfloor \nu \cdot N \rfloor$ Euclidean nearest neighbors with the same cluster label as $\x_0$.
$N$ is the number of points in the split.
We use $\nu = 0.1$ throughout so that each flow chart sees $10\%$ of the data in the split.
The base points and endpoints are drawn from the same data split.

\subsubsection{Flow chart architecture and training details}
For the MFM comparison and the flow chart preconditioner, we use the four-layer architecture of \Cref{sec:amortized-methods}.
We train these flow charts for at most $2{,}000{,}000$ steps, with early stopping on the data geometries.
On the data geometries, the learning rate remains constant after warmup, and training stops after $100$ validation checks ($50{,}000$ steps) without improvement.
For the capacity sweep, we use the training procedure of \Cref{app:capacity}, without early stopping.
We train with a batch containing $16$ base points, with $32$ endpoints for each.
During evaluation, we measure the average loss over the base points of the validation set, each with a single endpoint drawn from the same procedure with a fixed random seed.

We set $\gamma = 0.1\,\bar{r}/\sqrt{n}$, where $\bar{r}$ is the median Euclidean distance from a base point to its $k$-th nearest neighbor within its cluster, which is the radius of the region that endpoints are drawn from.
We empirically evaluate this over $64$ base points before training.

\subsubsection{Evaluation set generation}
This experiment evaluates different methods on triplets $(\x_0, V_x, \x_1)$, where $V_x \approx \log_{\x_0}(\x_1)$.
On the synthetic manifolds, the logmap is known in closed form.
On the data geometries, we estimate it using the initial velocity of the GP method of \Cref{sec:geodesic-interpolation-based-solvers}.
On each data geometry, we sample endpoint pairs according to the procedure in \Cref{app:neighbor-data-selection} over the test set.
We draw $B = 128$ base points uniformly from the test set and $M = 512$ endpoints for each, without replacement, under a fixed random seed, which gives $65{,}536$ triplets per dataset before the discards below.
The triplets are computed in FP64 before evaluation and saved to file, except that on the RMM geometries the network $M_{\varepsilon}^\theta$ and its derivatives are evaluated in FP32, while the GP iteration and all linear algebra remain in FP64 for numerical stability.
We found that in FP32, the GP method consistently diverged.
For all of the geometries, we use $K=50$ acceleration nodes, a length scale of $\sqrt{1/(2(K+1))}$, and an acceleration regularizer of $10^{-4}$.

Since the norm of the logmap solution is supposed to equal the geodesic distance between the base point and the query, the logmap solution can be refined by estimating the length and rescaling.
Let $L^{(i)} = \int_0^1\|\dot{\x}^{(i)}_t\|_g \, dt$ be the length of the GP curve, which we estimate with Simpson's rule on $101$ evenly spaced times.
The improved solution is
\begin{align}
    \label{eq:logmap-refinement}
    V_x^{(i)} = L^{(i)}\frac{\dot{\x}_0^{(i)}}{\|\dot{\x}_0^{(i)}\|_g}
\end{align}

Once the curve is fit, we save this refined logmap solution as $V_x$, together with the entire curve so that we can evaluate against it.
We do not include a solution if the GP did not converge, or if (for the synthetic datasets) the GP crosses low data density regions, which we measure with the RBF metric's network $h$ from \Cref{sec:rbf-metric}, keeping a curve only if $h \geq 0.5$ at $101$ evenly spaced times along it.
The GP converges when every squared fixed point residual falls below $10^{-4}$ within $500$ iterations.
\Cref{tab:ablation-dataset-sizes} lists the number of retained test triplets.
On these geometries, the errors measure agreement with the numerical reference, so we compare the GP baseline with the same solver using more acceleration nodes.

\paragraph{Efficient acceleration computation for GP on RMM cometrics}
For metric parameterized via \Cref{app:low-rank-rmm-parameterization}, solving for geodesic interpolations for the GP is expensive because every acceleration evaluation differentiates through a cometric solve.
However, by taking advantage of the structure of the metric, we can avoid derivative through the inverse.
At each of its nodes, the GP evaluates the geodesic acceleration $\ddot{\x} = -\Gamma_x(\x)(\dot{\x}, \dot{\x})$ from the position $\x$ and velocity $\dot{\x}$ of the curve there.
If it is computed using the expression from \Cref{eq:contracted-christoffel}, this differentiates the metric $g_{\varepsilon,x}^\theta(\x)$, which is the inverse of the cometric.

Instead, we compute the acceleration from Hamilton's equations.
Let $\p = g_{\varepsilon,x}^\theta(\x)\dot{\x}$ be the momentum.
Under the low rank parameterization, Hamilton's equations for $H(\x, \p) = \frac{1}{2}\p^\top g_{\varepsilon,x}^\theta(\x)^{-1}\p$ are
\begin{align}
    \dot{\x} = \delta\p + (1 - \delta)M_{\varepsilon}^\theta(\x)^\top M_{\varepsilon}^\theta(\x)\p, \qquad \dot{\p} = -\frac{1 - \delta}{2}\nabla_{\x}\left\|M_{\varepsilon}^\theta(\x)\p\right\|_2^2
\end{align}
where no gradient flows through $\p$.
Differentiating $\dot{\x}$ in time gives
\begin{align}
    \ddot{\x} = (1 - \delta)\left(D_{\dot{\x}}M_{\varepsilon}^\theta(\x)^\top M_{\varepsilon}^\theta(\x) + M_{\varepsilon}^\theta(\x)^\top D_{\dot{\x}}M_{\varepsilon}^\theta(\x)\right)\p + g_{\varepsilon,x}^\theta(\x)^{-1}\dot{\p}
\end{align}
where $D_{\dot{\x}}M_{\varepsilon}^\theta(\x)$ is the directional derivative of $M_{\varepsilon}^\theta(\x)$ at $\x$ along $\dot{\x}$.
This expression of the acceleration is cheap to compute because it requires a single linear solve, for $\p$, which is never differentiated.
Every derivative can be computed cheaply with Jacobian vector products of vector Jacobian products.

\subsubsection{Methods}
The baselines simulate the geodesic with initial position $\x_0$ and velocity $V_x$ by integrating Hamilton's equations for Hamiltonian $H(\x, \p) = \frac{1}{2}\p^\top g_x(\x)^{-1}\p$,
\begin{align}
    \label{app:eq-hamiltons-equations}
    \dot{\x}_t = g_x(\x_t)^{-1}\p_t, \qquad \dot{\p}_t = -\frac{1}{2}\nabla_{\x}\left(\p_t^\top g_x(\x_t)^{-1}\p_t\right), \qquad \p_0 = g_x(\x_0)V_x
\end{align}
from $t = 0$ to $t = 1$, so that $\x_1 = \exp_{\x_0}(V_x)$.
On the RMM geometry, we directly parameterize the cometric, so each update does not require a matrix inversion.
Only the initial momentum $\p_0$ requires an inversion.
We integrate with four fixed-step solvers from diffrax \citep{kidger2021on} of increasing order, Euler, Heun, Tsit5 \citep{tsitouras2011runge}, and Dopri8 \citep{prince1981high}, which cost one, two, six, and thirteen evaluations of \Cref{app:eq-hamiltons-equations} per step, plus one more at the first step for Tsit5 and Dopri8 and one evaluation of the cometric at $\x_0$ to form $\p_0$.
We run each solver with $S$ equal steps over $[0, 1]$.
For a solver that costs $c$ evaluations per step, we sweep $5$ step counts spaced geometrically from $1$ to $\lfloor 128/c \rfloor$ and rounded, which gives $S \in \{1, 3, 11, 38, 128\}$ for Euler, $\{1, 3, 8, 23, 64\}$ for Heun, $\{1, 2, 5, 10, 21\}$ for Tsit5, and $\{1, 2, 3, 5, 9\}$ for Dopri8.
These choices are so that no baseline evaluates the cometric more than $129$ times.
On the data geometries, we also use the solver's interpolation to record the state at the $50$ times used for the reference curve.
A query fails if the solver returns a nonfinite endpoint.
The solvers are run in FP32.

The flow chart computes $\exp_{\x_0}(V_x) = f^\theta_{\x_0}\!\left(Dh^\theta_{\x_0}(\x_0)V_x\right)$ as described in \Cref{sec:background-riemann-normal-coordinates}.

\subsubsection{Evaluation metrics}

For each method, we compute $\Eexp(\x_0)$ and $\Rexp(\x_0)$ as defined in \Cref{sec:expmap}.
The numerator uses the queries in $\mathcal I(\x_0)$ for which the method returns a finite endpoint.
The reference scale $r_{\mathrm{chord}}(\x_0)$ uses all queries in $\mathcal J(\x_0)$ and is shared across methods.
We average $\Rexp(\x_0)$ over base points with valid readings and report its mean and standard error multiplied by $100$.

We time every method with the procedure of \Cref{app:timing-details}.
Let $\tau(\x_0)$ be the wall clock time to compute the endpoints at a test base point $\x_0$, divided by the number of queries, so that $\tau(\x_0)$ is a time per query.
Before discarding queries without a valid reference solution, there are $1{,}024$ queries per base point on the closed form manifolds and $512$ on the data geometries.
For the flow chart, $\tau(\x_0)$ includes encoding $\x_0$ once, and the chart reuses this encoding for all of its queries.
For the ODE solvers, $\tau(\x_0)$ includes forming the initial momenta.
Every method evaluates the queries of a base point in parallel on a GPU, and we report $\tau$, the mean of $\tau(\x_0)$ over test base points.

\subsection{Capacity and fixed base point experiments}
\label{app:ablations}
\label{app:capacity}

We vary the capacity of the flow chart and compare amortized and non-amortized training on the closed form manifolds, the three synthetic datasets, and alanine dipeptide.
We use both the low and high dimensional closed form manifolds, where the logarithmic and exponential maps are known in closed form.
On the data geometries, we use the numerical references of \Cref{app:exp-map-experiments}.
See \Cref{sec:synthetic-metrics}, \Cref{sec:rbf-metric} and \Cref{app:datasets-and-rmm-geometries} for a description of each dataset, the geometry they define, and how the training, validation and test sets are generated.
Each dataset produces a set of $B$ base points, and for each base point, produces $N$ query points from which to train or evaluate with.
An element of each dataset consists of a triplet of the form $(\x_0, V_x, \x)$, where $V_x = \log_{\x_0}(\x)$ and $\x = \exp_{\x_0}(V_x)$ are the solution to the logarithmic and exponential maps respectively.

We train flow charts on this data in two different settings.
The first is an amortized setting over base points, where we train on $B_{\text{train}}$ base points for each dataset, each with $N_{\text{train}}$ associated query points.
During evaluation, we compute $\Rlog$ and $\Rexp$ over a set of $B_{\text{test}}$ held out test base points, each with $N_{\text{test}}$ unseen query points before discarding those without a valid reference solution.

The second is a non-amortized setting where we train a flow chart at a single test base point.
These $B_{\text{single}}=4$ test base points are selected at random (with a fixed random seed) from the $B_{\text{test}}$ held out test base points from each dataset.
For each of these base points, we construct a training set of $N_{\text{single}}$ query points that is separate from its $N_{\text{test}}$ test query points.
On the closed form datasets, we sample $N_{\text{single}}$ new query points at the base point with the sampling procedure of \Cref{sec:synthetic-metrics}, independently of the test query points.
On the synthetic and bio datasets, the training query points are the $N_{\text{single}}$ nearest neighbors of the base point in the training set, selected as in \Cref{app:neighbor-data-selection}.
The base point is a test point and its training query points are training points, so a flow chart never trains on a test query point.
\Cref{tab:ablation-dataset-sizes} lists $N_{\text{single}}$ for each dataset.

For \Cref{fig:accuracy}, we compute $\Rlog$ at each of these four base points using the same test queries as for the amortized flow chart.

\begin{table}[t]
    \centering
    \caption{Sizes of the training and test sets.
    On the synthetic and bio datasets, every point of the training set is a base point and its query points are its $\lfloor 0.1\,B_{\text{train}} \rfloor$ nearest neighbors with the same cluster label.
    The last column is the number of test triplets that remain out of $B_{\text{test}} N_{\text{test}}$ after the discards of \Cref{app:exp-map-experiments}.}
    \label{tab:ablation-dataset-sizes}
    {
    \small
    \setlength{\tabcolsep}{5pt}
    \begin{tabular}{l r r r r r r r}
        \toprule
        Dataset & $n$ & $B_{\text{train}}$ & $N_{\text{train}}$ & $B_{\text{test}}$ & $N_{\text{test}}$ & $N_{\text{single}}$ & Test triplets \\
        \midrule
        Sphere & $2\,/\,30$ & $2{,}048$ & $1{,}024$ & $128$ & $1{,}024$ & $262{,}144$ & $131{,}072$ \\
        Flat torus & $2\,/\,30$ & $2{,}048$ & $1{,}024$ & $128$ & $1{,}024$ & $262{,}144$ & $131{,}072$ \\
        Hyperbolic & $2\,/\,30$ & $2{,}048$ & $1{,}024$ & $128$ & $1{,}024$ & $262{,}144$ & $131{,}072$ \\
        SPD & $3\,/\,28$ & $2{,}048$ & $1{,}024$ & $128$ & $1{,}024$ & $262{,}144$ & $131{,}072$ \\
        Grassmann & $4\,/\,30$ & $2{,}048$ & $1{,}024$ & $128$ & $1{,}024$ & $262{,}144$ & $131{,}072$ \\
        Rotations & $3\,/\,28$ & $2{,}048$ & $1{,}024$ & $128$ & $1{,}024$ & $262{,}144$ & $131{,}072$ \\
        \midrule
        Disk & $2$ & $350{,}000$ & $35{,}000$ & $128$ & $512$ & $35{,}000$ & $65{,}536$ \\
        Two moons & $2$ & $350{,}000$ & $35{,}000$ & $128$ & $512$ & $35{,}000$ & $62{,}221$ \\
        Nested circles & $2$ & $350{,}000$ & $35{,}000$ & $128$ & $512$ & $35{,}000$ & $65{,}041$ \\
        \midrule
        Alanine dipeptide & $30$ & $525{,}000$ & $52{,}500$ & $128$ & $512$ & $52{,}500$ & $65{,}529$ \\
        \bottomrule
    \end{tabular}
    }
\end{table}

\subsubsection{Flow-chart architecture}
\label{app:capacity-architecture}
In all of these experiments, we use a modified version of the conditional neural spline flow of \Cref{sec:amortized-methods}.
We keep the conditioned elementwise scale and shift and the invertible linear map that are placed before each coupling layer, and change the coupling layer and the width of its conditioner.
Each coupling layer transforms an input $v$ into $s\cdot \mathrm{spline}(v) + t$, where the spline is a rational quadratic spline with $8$ bins on $[-10, 10]$, and the $23$ spline parameters, the scale $s$ and the shift $t$ are all predicted by a FiLM conditioned residual network with $3$ blocks and hidden width of $64$.
The scale $s$ is passed through a softplus so that the flow chart is close to the identity transformation at initialization.

In the amortized setting, the base point encoder is a residual network with $8$ gated residual blocks of width $64$ that encodes the base point into a $64$-dimensional feature vector for all coupling layers.
In the non-amortized setting, the base point is fixed, so the base point encoder has a single gated residual block.
We vary the capacity of the amortized flow chart through the number of coupling layers, $L \in \{4, 8, 16, 32\}$, and use $L=32$ for the accuracy comparison at fixed base points.

\subsubsection{Training}
We train every flow chart with the procedure of \Cref{sec:amortized-methods}, with a budget of $2{,}000{,}000$ steps, learning rate decay and no early stopping.
All non-amortized flow charts reported here and all amortized flow charts with $L=32$ reached the full budget.
On two moons, the amortized fits with $L=4$ and $L=8$ diverged after $1{,}154{,}500$ and $1{,}292{,}500$ steps respectively.
We evaluate these flow charts using their best checkpoints selected before divergence.
On the geometries whose cometric is learned with RMM, we fix $\varepsilon = 0.1$.
In both settings, we set $\gamma = 0.1\,r_{\text{max}}/\sqrt{n}$ on the closed form datasets as in \Cref{sec:amortized-methods}, and $\gamma = 0.1\,\bar{r}/\sqrt{n}$ on the synthetic and bio datasets as in \Cref{app:exp-map-experiments}.
In the amortized setting, each training batch contains $16$ base points with $32$ query points each.
In the non-amortized setting there is a single base point, so each training batch contains $512$ query points drawn uniformly from the $N_{\text{single}}$ training points of that base point.
On the closed form datasets, the validation loss in the non-amortized setting is evaluated on a validation set containing $1{,}024$ independently sampled points.
On the synthetic and bio datasets, this validation set consists of the $\lfloor 0.1\,N_{\text{val}} \rfloor$ nearest neighbors of the base point in the validation set, selected as in \Cref{app:neighbor-data-selection}, where $N_{\text{val}}$ is the number of points in the validation set.

\subsubsection{Evaluation}
We evaluate every flow chart in the same way as the flow charts of the experiments above.
On every geometry, we compute the mean of $\Rlog$ or $\Rexp$ across base points as described in \Cref{app:synthetic-experiment-eval-metrics,app:exp-map-experiments}.
On the synthetic and bio datasets, the ground truth normal coordinates of a test triplet are $\s = g_x(\x_0)^{\frac{1}{2}}V_x$.

In the amortized setting, we evaluate the flow chart at all $B_{\text{test}}$ test base points.
In the non-amortized setting, we evaluate each flow chart at its own base point and report the mean over the $B_{\text{single}}$ flow charts.
For the accuracy comparison at fixed base points, we evaluate the flow charts and numerical baselines at the same $B_{\text{single}}$ base points.

\subsubsection{Baselines}
\label{app:ablation-baselines}

\paragraph{Closed form manifolds ODE Solvers}
We run the ODE solvers of \Cref{app:exp-map-experiments} on the closed form manifolds in both low and high dimensions.
All $N_{\text{test}}$ queries of a test base point are solved in a single batch.

\paragraph{Closed form manifolds MFM and GP, GEORCE}
See \Cref{sec:normal-recovery-protocol}.

\paragraph{Data manifolds ODE Solvers}
See \Cref{app:exp-map-experiments}.

\paragraph{Data manifolds MFM and GP, GEORCE}
We run the methods of \Cref{sec:geodesic-interpolation-based-solvers} on the synthetic and bio datasets, and evaluate them as in \Cref{app:synthetic-experiment-eval-metrics} on the test triplets of \Cref{app:exp-map-experiments} that were not discarded.
The batch sizes of the GP and GEORCE are selected from $\{64, 128, 256, 512\}$.

The two MFMs are matched to the flow chart of \Cref{app:exp-map-experiments} on each dataset, with the time measured on the $512$ queries of a test base point, and \Cref{tab:ablation-mfm-sizes} lists their sizes.
They are trained on the same training pairs as that flow chart, with the same training procedure as \Cref{sec:normal-recovery-protocol}.
We train these MFMs for $2{,}000{,}000$ steps with learning rate decay and no early stopping.

\begin{table}[t]
    \centering
    \caption{Sizes of the MFM baselines on the synthetic and bio datasets.}
    \label{tab:ablation-mfm-sizes}
    {
    \footnotesize
    \setlength{\tabcolsep}{3pt}
    \begin{tabular}{l r r r r r r r r r}
        \toprule
        & & \multicolumn{2}{c}{Flow chart} & \multicolumn{3}{c}{MFM, parameter matched} & \multicolumn{3}{c}{MFM, time matched} \\
        \cmidrule(lr){3-4} \cmidrule(lr){5-7} \cmidrule(lr){8-10}
        Dataset & $n$ & Parameters & $\mu$s & Size & Parameters & $\mu$s & Size & Parameters & $\mu$s \\
        \midrule
        Disk & $2$ & $187{,}504$ & $1.85$ & $64 \times 15$ & $188{,}674$ & $2.20$ & $64 \times 12$ & $151{,}042$ & $1.82$ \\
        Two moons & $2$ & $187{,}504$ & $1.92$ & $64 \times 15$ & $188{,}674$ & $2.23$ & $64 \times 12$ & $151{,}042$ & $1.87$ \\
        Nested circles & $2$ & $187{,}504$ & $1.89$ & $64 \times 15$ & $188{,}674$ & $2.22$ & $64 \times 12$ & $151{,}042$ & $1.87$ \\
        \midrule
        Alanine dipeptide & $30$ & $305{,}008$ & $2.34$ & $128 \times 6$ & $309{,}790$ & $2.74$ & $64 \times 7$ & $93{,}726$ & $2.29$ \\
        \bottomrule
    \end{tabular}
    }
\end{table}

\subsubsection{Experiments}

This experiment measures how the errors of a flow chart depend on its capacity, and how much of the error is due to amortization.
For every dataset of \Cref{tab:ablation-dataset-sizes}, we use an amortized flow chart at each $L\in\{4,8,16,32\}$ for the runtime comparisons.
For the accuracy comparison at fixed base points, we use $B_{\text{single}}=4$ non-amortized flow charts with $L=32$, one at each base point.
Every other choice is held fixed across $L$, and a flow chart in the non-amortized setting differs from the amortized flow chart of the same $L$ in its base point encoder and its training data alone.
Each flow chart is trained once, from one seed.

We compare against the baselines of the experiments above, evaluated on the same test triplets.
For the logarithmic map, these are the GP, GEORCE and the two MFMs of \Cref{sec:geodesic-interpolation-based-solvers}, and for the exponential map they are the ODE solvers of \Cref{app:exp-map-experiments}.
The two MFMs are selected to have the same parameter count and time as the flow chart used in the original experiment, and not the size of the ablation models.

\Cref{tab:ablation-parameters} contains the parameter count of the amortized flow chart for each dimension and $L$.

\begin{table}[t]
    \centering
    \caption{Parameter counts of the amortized flow charts used in the capacity sweep for each number of coupling layers $L$.}
    \label{tab:ablation-parameters}
    {
    \footnotesize
    \setlength{\tabcolsep}{3pt}
    \begin{tabular}{l r r r r r}
        \toprule
        Datasets & $n$ & $L = 4$ & $8$ & $16$ & $32$ \\
        \midrule
        Sphere, torus, hyperbolic & $2$ & $319{,}232$ & $533{,}760$ & $962{,}816$ & $1{,}820{,}928$ \\
        SPD, rotations & $3$ & $326{,}616$ & $548{,}464$ & $992{,}160$ & $1{,}879{,}552$ \\
        Grassmann & $4$ & $334{,}016$ & $563{,}200$ & $1{,}021{,}568$ & $1{,}938{,}304$ \\
        SPD, rotations & $28$ & $516{,}416$ & $926{,}464$ & $1{,}746{,}560$ & $3{,}386{,}752$ \\
        Sphere, torus, hyperbolic, Grassmann & $30$ & $532{,}032$ & $957{,}568$ & $1{,}808{,}640$ & $3{,}510{,}784$ \\
        \midrule
        Disk, two moons, nested circles & $2$ & $319{,}232$ & $533{,}760$ & $962{,}816$ & $1{,}820{,}928$ \\
        \midrule
        Alanine dipeptide & $30$ & $532{,}032$ & $957{,}568$ & $1{,}808{,}640$ & $3{,}510{,}784$ \\
        \bottomrule
    \end{tabular}
    }
\end{table}

\subsubsection{Scaling with dimension}
\label{app:dimension-scaling}
\Cref{fig:scaling} shows that the cost of a pretrained flow chart also scales favorably with dimension compared with the numerical baselines.
Across the closed form manifolds that have low and high dimensional variants, the smallest and largest chart sizes only take $1.30$ to $1.63$ times longer to run in high dimensions while the GP grows by $3.0$ to $52\times$ and GEORCE by $3.8$ to $27\times$ (at $K=16$).

\begin{figure}[H]
\centering
\includegraphics[width=0.5\textwidth,trim=9bp 7bp 3bp 5.5bp,clip]{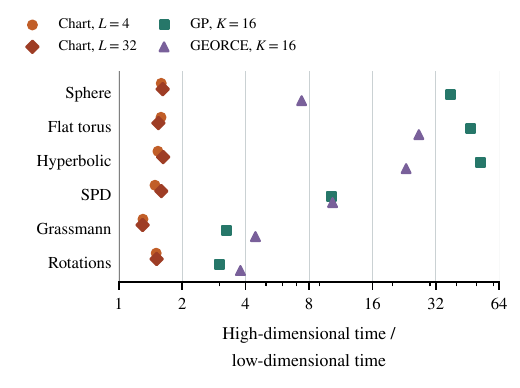}
\caption{Increase in time per logmap query from the low dimensional to the high dimensional version of each manifold.
Each method uses the same architecture or knot count in both dimensions.}
\label{fig:scaling}
\end{figure}

\subsection{Matching and solver resolution}
\label{app:matching}
To compare runtimes at matched accuracy, we set the error threshold $\mathcal{R}_\star$ to the larger of the minimum mean errors in the two sweeps.
The mean error is the mean of $\Rlog$ or $\Rexp$ across base points for the logarithmic or exponential map respectively.
Within each sweep, we select the fastest setting whose mean error is at most $\mathcal{R}_\star$.
We report the speedup as the runtime of the baseline divided by that of the flow chart.
This comparison uses the test errors to select settings and describes their measured performance.
We do not interpolate between settings.
To compare errors at matched cost, we fix the $L=32$ flow chart and select the ODE setting with the lowest error among those whose runtime does not exceed that of the chart.
\Cref{tab:matched-settings} lists the settings selected for the comparisons in \Cref{fig:speed}.

On the two dimensional sphere, several ODE settings have very large finite errors, and the errors do not consistently decrease as the number of steps increases.
The matched comparison in \Cref{fig:speed} uses Heun with $64$ steps.
These results show instability on this geometry, but we have not identified its cause.
The conclusions about the exponential map in high dimensions do not depend on this geometry.

\begin{table}[H]
\centering
\caption{Settings selected for the runtime comparisons in \Cref{fig:speed}.
For each numerical method, $L$ is the number of coupling layers in the chart used in that comparison.
$K$ is the number of knots for GP or GEORCE, and $S$ is the number of ODE steps.
Each comparison selects the fastest chart and numerical settings whose mean normalized RMSE is at most $\mathcal{R}_\star$.}
\label{tab:matched-settings}
\small
\begin{tabular*}{\textwidth}{@{\extracolsep{\fill}}lrrrrrlr@{}}
\toprule
& \multicolumn{4}{c}{Logarithmic map} & \multicolumn{3}{c}{Exponential map} \\
\cmidrule(lr){2-5}\cmidrule(l){6-8}
& \multicolumn{2}{c}{GP} & \multicolumn{2}{c}{GEORCE} & \multicolumn{3}{c}{ODE} \\
\cmidrule(lr){2-3}\cmidrule(lr){4-5}\cmidrule(l){6-8}
Geometry & $L$ & $K$ & $L$ & $K$ & $L$ & Solver & $S$ \\
\midrule
Sphere 2 & 16 & 16 & 4 & 32 & 16 & Heun & 64 \\
Flat torus 2 & 8 & 16 & 4 & 32 & 16 & Tsit5 & 2 \\
Hyperbolic 2 & 8 & 16 & 4 & 32 & 16 & Tsit5 & 2 \\
SPD 3 & 16 & 16 & 4 & 32 & 32 & Dopri8 & 1 \\
Grassmann 4 & 32 & 16 & 4 & 32 & 32 & Dopri8 & 1 \\
Rotations 3 & 32 & 16 & 4 & 32 & 32 & Tsit5 & 2 \\
\midrule
Sphere 30 & 32 & 8 & 32 & 32 & 32 & Heun & 8 \\
Flat torus 30 & 32 & 16 & 16 & 32 & 32 & Dopri8 & 2 \\
Hyperbolic 30 & 32 & 8 & 32 & 32 & 32 & Tsit5 & 2 \\
SPD 28 & 32 & 8 & 32 & 32 & 32 & Heun & 8 \\
Grassmann 30 & 32 & 8 & 32 & 32 & 32 & Dopri8 & 1 \\
Rotations 28 & 32 & 8 & 32 & 16 & 32 & Tsit5 & 2 \\
\midrule
Disk & 8 & 16 & 4 & 32 & 4 & Dopri8 & 1 \\
Two moons & 16 & 16 & 4 & 32 & 16 & Tsit5 & 2 \\
Nested circles & 4 & 16 & 4 & 32 & 16 & Dopri8 & 2 \\
Alanine dipeptide & 32 & 8 & 32 & 32 & 32 & Heun & 3 \\
\bottomrule
\end{tabular*}
\end{table}

\subsubsection{Exponential map errors and runtime budget}
\label{app:expmap-results}
For the exponential map, the flow chart with $L=32$ achieves a mean $100\times\Rexp$ between $0.15$ and $2.51$ across the seven geometries with $n\geq28$, including $1.99\pm0.29$ on alanine dipeptide.
We also compare flow charts and ODE solvers when the ODE solvers are limited to the runtime of the $L=32$ chart.
Within this budget, the mean $\Rexp$ of the most accurate solver is $6.1$--$195$ times that of the flow chart on the high dimensional closed form manifolds, and $3.5$ times that of the flow chart on alanine dipeptide.
The solvers can reach lower errors if given more time.
\Cref{fig:full-expmap} gives the errors and runtimes at every chart size and ODE step count.

\subsection{Flow chart preconditioner}
\label{app:chart-preconditioner}
For \Cref{tab:preconditioner}, we use the pretrained flow charts with four coupling layers described in \Cref{app:models}.
For each base point, we express the metric in the chart coordinates using \Cref{eq:metric-change-of-variables}, encode each query, and run GEORCE between the origin and the encoded query.
We map the initial velocity back to data coordinates through $Df_{\x_0}^{\theta}(\zero)$ to obtain the logmap estimate.
We use the same knot counts and stopping criterion as GEORCE in data coordinates.
We report the mean of $\Rlog$ across base points and the ratio of the mean errors at each knot count.
GEORCE successfully returns a solution for every query in this comparison in both data and chart coordinates.

At $K=32$, the mean of $\Rlog$ is reduced by $2.3$--$3.6\times$ on the six high dimensional closed form manifolds, and by $2.2$--$5.9\times$ on the data geometries.
The improvement holds at all four knot counts, with factors of $2.2$--$8.1$ across the $40$ settings.
This suggests that the flow chart learns a space in which the numerical errors associated with GEORCE are reduced.
We compare the methods at equal knot counts, so these evaluations do not establish a runtime improvement.

\clearpage
\section{Complete experimental results}
\label{app:complete-results}
We summarize the results over all geometries and tested settings in three figures.
\Cref{fig:full-logmap,fig:full-expmap} show the accuracy and time per query for the logarithmic and exponential maps respectively.
\Cref{fig:full-capacity} shows the effect of flow chart capacity and amortization.

\begin{figure}[H]
\centering
\includegraphics[width=\textwidth]{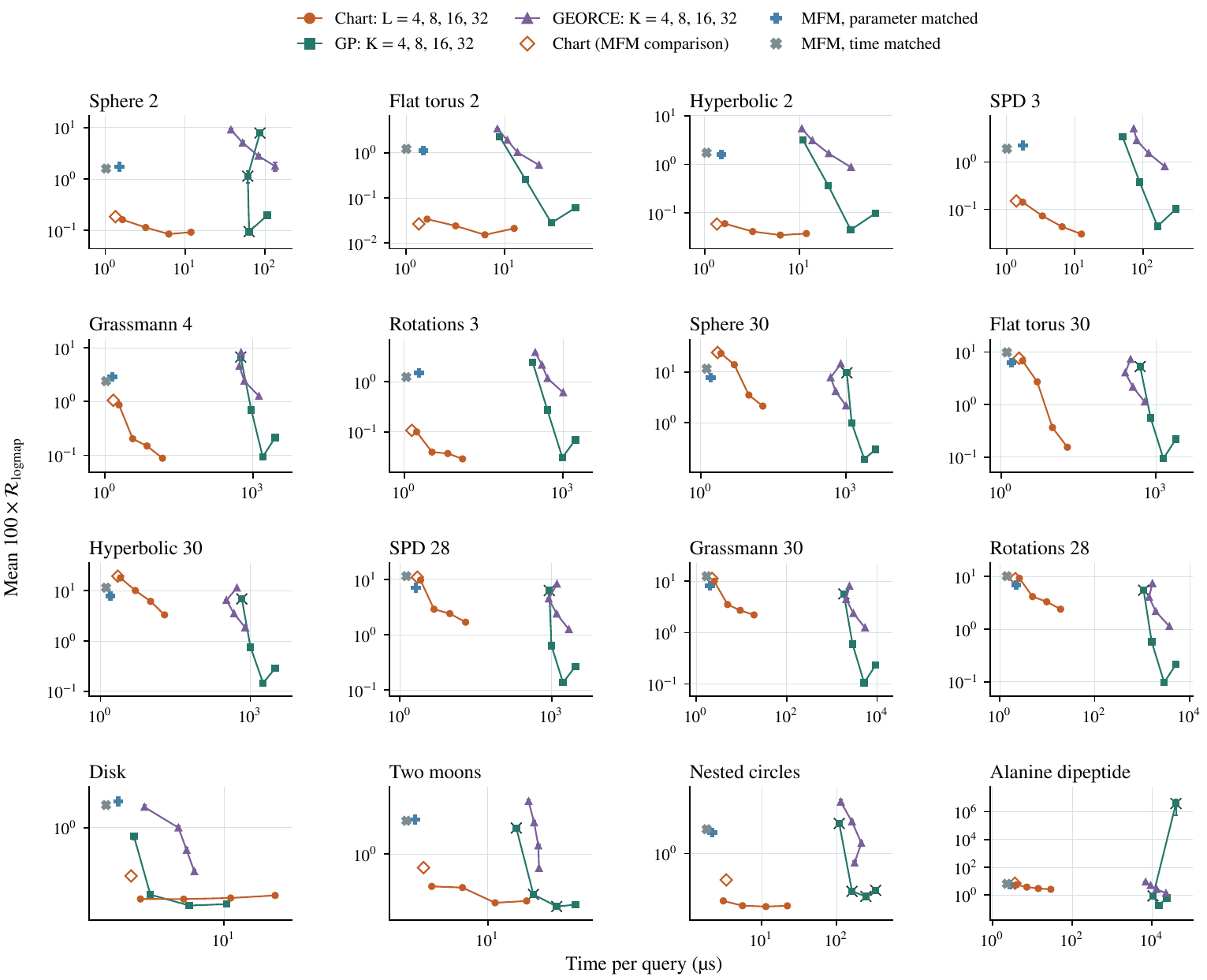}
\caption{Logarithmic map accuracy and runtime across all sixteen geometries.
Each point gives the mean of $100\times\Rlog$, with one standard error across base points.
The points for the four flow chart sizes and the four knot counts of the GP and GEORCE are joined by lines.
The open diamond denotes the smaller flow chart used for the MFM comparison, shown alongside both MFM baselines.
Both axes use logarithmic scales, with ranges chosen separately for each geometry.
A black cross marks a setting with failed queries, which are excluded when computing its error.}
\label{fig:full-logmap}
\end{figure}

\clearpage
\begin{figure}[H]
\centering
\includegraphics[width=\textwidth]{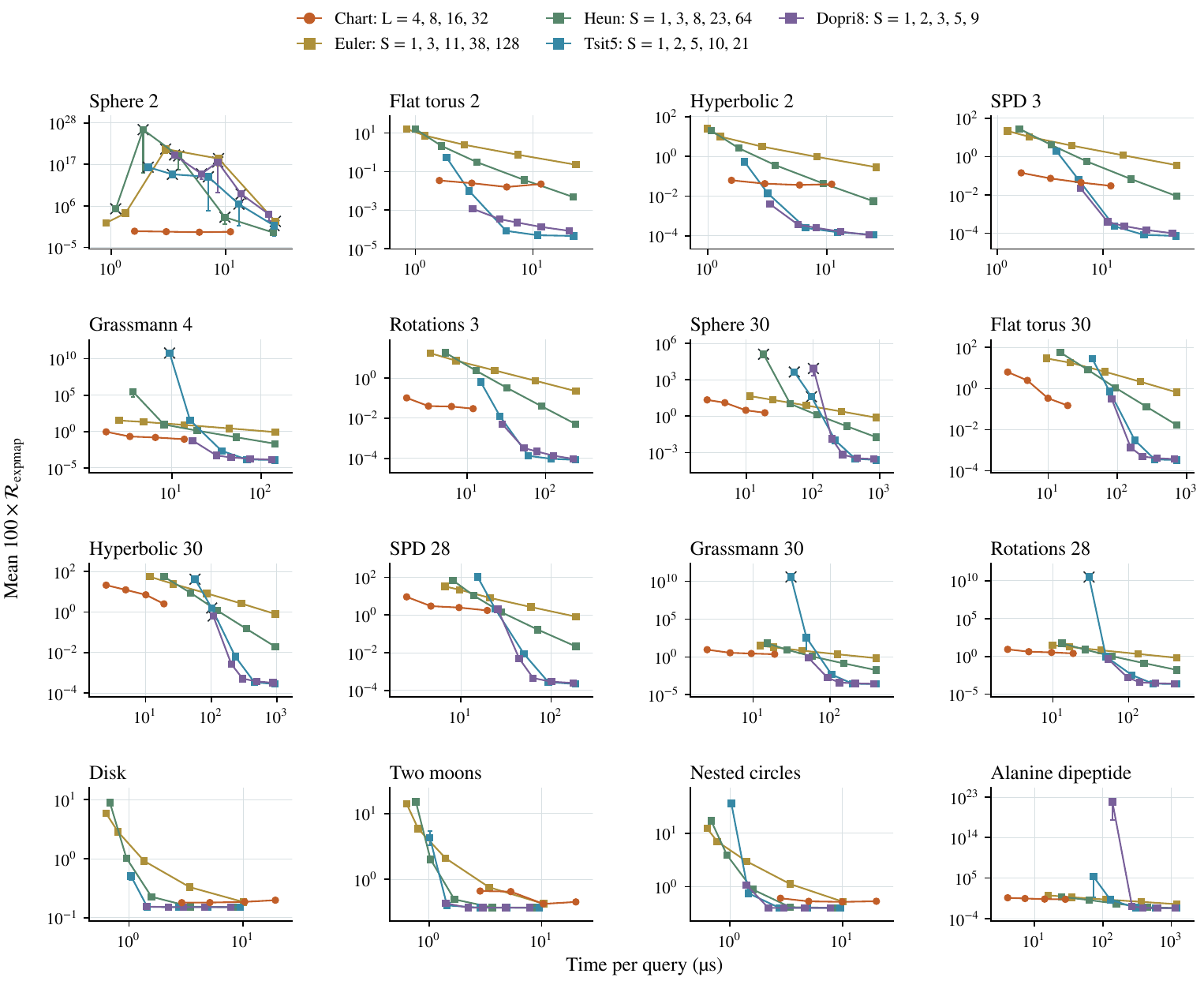}
\caption{Exponential map accuracy and runtime across all sixteen geometries.
Each point gives the mean of $100\times\Rexp$, with one standard error across base points.
The curves show every flow chart size and ODE step count listed in the legend.
Both axes use logarithmic scales, with ranges chosen separately for each geometry.
A black cross marks a setting with nonfinite endpoints, which are excluded when computing its error.
We retain all finite errors.
The large errors on the two dimensional sphere show that several ODE settings are unstable on this geometry.
All flow chart queries with valid references return finite endpoints.}
\label{fig:full-expmap}
\end{figure}

\clearpage
\begin{figure}[H]
\centering
\includegraphics[width=\textwidth]{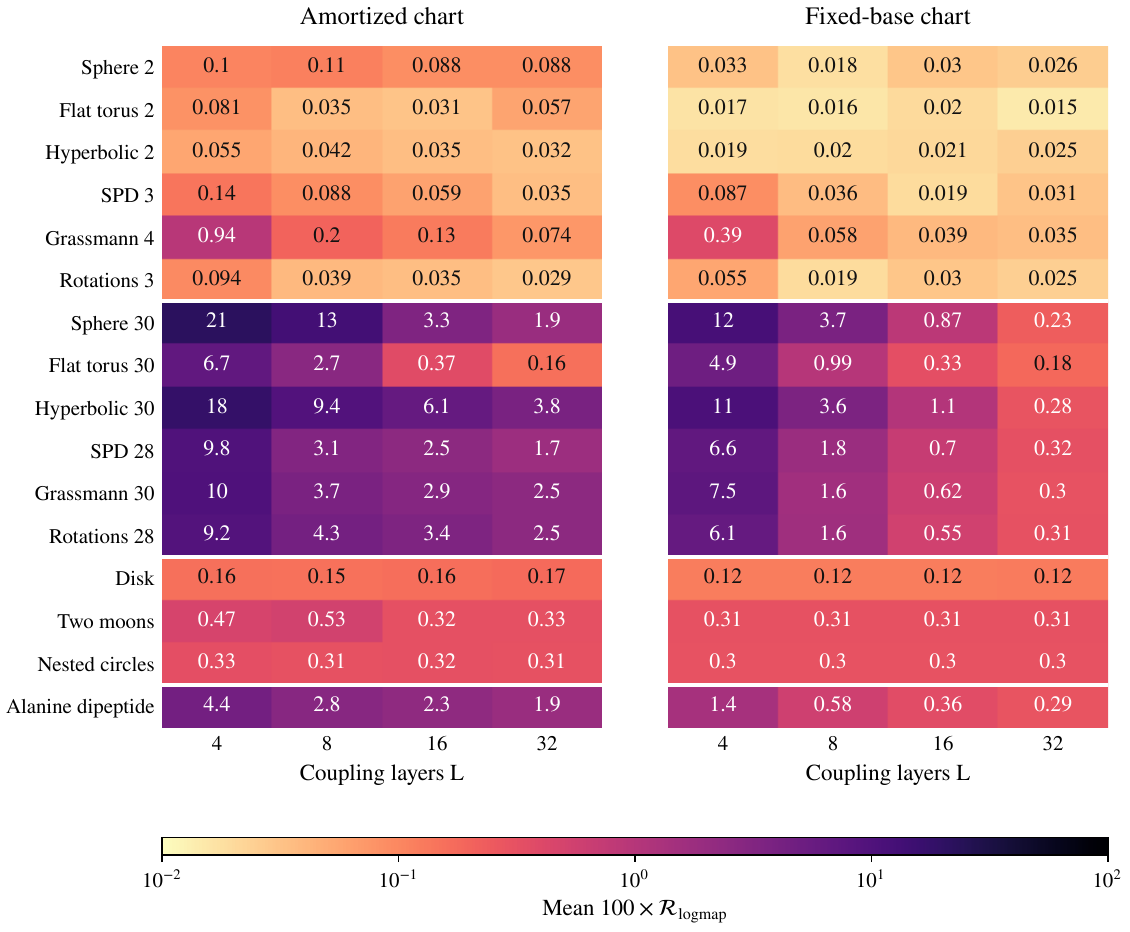}
\caption{Effect of flow chart capacity and amortization.
The panels show the mean of $100\times\Rlog$ for amortized and non-amortized flow charts, evaluated on the same four test base points at each chart size.
The two panels share a logarithmic color scale.}
\label{fig:full-capacity}
\end{figure}

\end{document}